\documentclass[11pt]{article}

\RequirePackage{natbib}
\RequirePackage[OT1]{fontenc}
\RequirePackage{amsthm,amsmath}
\RequirePackage[colorlinks,citecolor=blue,urlcolor=blue]{hyperref}

\newtheorem{thm}{Theorem}[section]
\newtheorem{cor}[thm]{Corollary}
\newtheorem{lem}[thm]{Lemma}
\newtheorem{prop}[thm]{Proposition}
\newtheorem{defn}[thm]{Definition}

\newtheorem{aspt}[thm]{Assumption}
\newtheorem{rem}[thm]{Remark}

\usepackage{amsmath,amssymb,amsfonts, amsthm, amsbsy, mathtools, epsfig}
\usepackage[usenames]{color}
\usepackage{rotating}
\usepackage{boxedminipage}
\usepackage[colorlinks,citecolor=blue,urlcolor=blue]{hyperref}
\usepackage{epstopdf}
\usepackage{natbib}
\usepackage{enumerate}

\usepackage{verbatim}
\usepackage{graphicx}
\usepackage{caption}
\usepackage{subcaption}
\usepackage{bbm}
\usepackage{xcolor}
\usepackage{bm}
\usepackage{mathrsfs,color,dsfont}
\usepackage{algorithm}
\usepackage{thmtools}
\usepackage{thm-restate}
\usepackage{hyperref}
\usepackage{cleveref}

\newcommand{\calP}{{\mathcal{P}}}

\newcommand{\RNum}[1]{\uppercase\expandafter{\romannumeral #1\relax}}
\makeatother
\newcommand{\argmin}{\mathop{\mathrm{arg\,min}{}}}

\newcommand{\tr}{\text{tr}}

\def\argmin{\mathop{\rm argmin}}

\newcommand\tabcaption{\def\@captype{table}\caption}

\newcommand\blue{}

\definecolor{DSgray}{cmyk}{0,1,0,0}

\newcommand{\E}{\mathbb{E}}

\newcommand{\calD}{\mathcal{D}}

\newcommand{\reals}{\mathbb{R}}
\newcommand{\Prob}{\mathbf{P}}

\newcommand{\Fhat}{\widehat{F}}

\newcommand{\tildeO}{\widetilde{O}}

\begin{document}

\title{Dimension Independent Generalization Error by Stochastic Gradient Descent}

\author{Xi Chen\footnote{Stern School of Business, New York University, Email: xichen@nyu.edu}  ~ Qiang Liu\footnote{Department of Mathematics, National University of Singapore, Email: matliuq@nus.edu.sg} ~Xin T. Tong\footnote{Department of Mathematics, National University of Singapore, Email: mattxin@nus.edu.sg}}
\date{\today}

\maketitle

\begin{abstract}

One classical canon of statistics is that large models are prone to overfitting, and model selection procedures are necessary for high dimensional data.  However, many overparameterized models, such as neural networks, perform very well in practice, although they are often trained with simple online methods and regularization.  The empirical success of overparameterized models, which is often known as benign overfitting, motivates us to have a new look at the statistical generalization theory for online optimization.  In particular, we present a general theory on the generalization error of stochastic gradient descent (SGD) solutions for both convex and \blue{locally convex} loss functions. We further discuss data and model conditions that lead to a ``low effective dimension". Under these conditions, we show that the generalization error either does not depend on the ambient dimension $p$ or depends on $p$  via a poly-logarithmic factor. We also demonstrate that in several widely used statistical models,  the ``low \blue{effective} dimension'' arises naturally in overparameterized settings. The studied statistical applications include both convex models such as linear regression and logistic regression and non-convex models such as $M$-estimator and two-layer neural networks.

\end{abstract}


\section{Introduction}


The study of overfitting phenomenon has been an important topic in statistics and machine learning. From classical statistical learning theory,  we understand that when the number of model parameters is large compared to the amount of data, the generalization error can be excessively large even if  the training error is small. This phenomenon is usually known as overfitting. For this reason, dimension reduction or feature selection mechanisms such as principle component analysis (PCA) and shrinkage methods are often required in the training phase to reduce model dimension and avoid overfitting.

In recent years, deep neural networks have achieved  great successes in practical applications. Researchers have found out that overparameterized neural networks usually achieve superior performance \citep{GRS:18,LL2018,NLBLS:19,ALL:19,ADHLW:19}. Moreover, these models are often trained with simple  regularization and do not need dimension reduction procedures. This phenomenon is sometimes referred to as \emph{benign overfitting} \citep{BLLT:19}.  To understand it, we need a new statistical framework to study generalization errors. 

Although there is much practical evidence on the benefit of overparameterization, the existing theoretical study mainly focuses on linear models (see, e.g., \citealp{BLLT:19, Nakkiran19, AKT:19}) or neural networks with certain special data structures (see, e.g., \citealp{LL2018}). The main purpose of our paper is to systematically investigate the generalization error for a risk minimization problem when the number of parameters $p$ is much larger than the sample size $N$. In particular, we establish a generalization error bound for  stochastic gradient descent (SGD) solutions for both convex (e.g., linear regression and logistic regression) and non-convex problems (some $M$-estimators and neural networks). We focus our study on the SGD algorithm because it has been widely used in large-scale data learning due to its computational and memory efficiency. 
%


Let us briefly introduce our setup of the overfitting problem and the SGD algorithm. We consider the following population risk minimization problem under a loss function $F$, which can be either convex or non-convex:
\begin{equation}\label{eq:sto_opt}
w^*=\argmin_w F(w),\quad F(w):=\E_\zeta f(w,\zeta).
\end{equation}
In \eqref{eq:sto_opt}, $w\in \reals^p$ is a $p$-dimensional parameter vector,  $\zeta$ denotes a random sample from a certain probability distribution,  and $f(\,\cdot\,,\zeta)$ is the loss function on each individual data $\zeta$. The global minimizer $w^*$ is often the true model parameter in  statistical estimation problems. In practice, the distribution of $\zeta$ is usually unknown, and one only has the access to $N$ \emph{i.i.d.} samples  $\zeta_1,\ldots,\zeta_N$ from the population. Instead of minimizing the population risk $F(w)$ in \eqref{eq:sto_opt}, it is more practical  to minimize the empirical loss function 
\begin{equation}\label{eq:empirical}
\hat{F}(w)=\frac1N \sum_{i=1}^N f(w,\zeta_i).
\end{equation}
Often, instead of directly minimizing the empirical loss, an extra regularization term is sometimes added to the empirical loss to avoid overfitting. In this paper,  we consider the most commonly used ridge or Tikhonov regularization. 
 The corresponding regularized empirical loss function takes the following form,
\begin{equation}\label{eq:reg}
\Fhat_\lambda(w):=\Fhat(w)+\dfrac{\lambda}{2} \|w\|^2=\frac1N\sum_{i=1}^N f_\lambda(w,\zeta_i),\quad  f_\lambda(w,\zeta):=f(w,\zeta)+\dfrac{\lambda}{2} \|w\|^2.
\end{equation}
The weight of regularization is controlled by $\lambda$, which is a tuning parameter. \blue{When $\lambda=0$, this corresponds to the ridgeless regression. We will study this setup in this paper, which corresponds to ``implicit regularization''. }One popular way to optimize $\Fhat_\lambda$ in machine learning is via SGD. In particular, for a generic initialization parameter  $w_0$, SGD is an iterative algorithm, where the $(n+1)$-th iterate $w_{n+1}$ is updated according to the following equation,
\begin{align}
\label{eq:sgd_iter}
w_{n+1}: =w_{n} -\eta \nabla f_\lambda(w_{n}, \zeta_n) =w_{n} -\eta \left( \nabla f(w_{n}, \zeta_n) +\lambda w_n \right).
\end{align}
By running through $N$ samples, SGD outputs the $N$-th iterate $w_N$ as the final estimator of $w^*$. 
\blue{Notably, SGD iterates are affected by the stochasticity of the data samples $\zeta$. To reduce such noise and improve accuracy, the averaged SGD (ASGD) method uses the average iterate 
\[
\bar{w}_N=\frac{1}{N}\sum_{i=1}^N w_i
\]
as the final estimator of $w^*$.}
In SGD iterations \eqref{eq:sgd_iter}, the hyper-parameter $\eta$ is known as the stepsize. In our paper, we consider using a constant stepsize, which is  a popular choice in practice \citep{BM11}. The value of $\eta$ will be discussed later in our theoretical results. 
Moreover, in \eqref{eq:sgd_iter}, the gradient is taken with respect to the parameter vector $w$. For notational simplicity, we will use ``$\nabla$" as a short notation for ``$\nabla_w$" throughout the paper.

When the sample size $N$ is much larger than the dimensionality  $p$, it is expected  that $w_N$ would be close to $w^*$ under certain conditions. However, in an overparameterized setting where $N$ is less than $p$,  the solution $w_N$ can be far away from $w^*$. In this case, estimating the underlying parameter accurrately usually requires strong assumptions. However, for many machine learning tasks, it is of more interest in achieving small generalization error, which is defined as follows,
\begin{equation}\label{eq:generalization}
G(w_N)=F(w_N)-F(w^*).
\end{equation}
The main purpose of the paper is to provide an upper bound of the generalization error in \eqref{eq:generalization} in  overparameterized settings. We will characterize the scenarios where such a generalization error bound is independent of $p$ or only involves in poly-logarithmic factors of $p$.


\subsection{Main results and paper organization}
The main message of this paper is as follows. For a large class of statistical learning problems where the \emph{effective dimension} is low (see the rigorous definition in Section  \ref{sec:nooverfit}), the stochastic gradient descent (SGD) algorithm with proper ridge regularization will not overfit even if the ambient model dimension is much larger than the sample size. In particular, we will show that the generalization error has at most poly-logarithmic dependence on the ambient model dimension $p$.

In  Section \ref{sec:generalize},  we present a framework for generalization error analysis. 
This framework is designed to handle 
\blue{ high dimensional regression problems. 
We will separately discuss two scenarios. In the first scenario, the true parameter is weakly sparse with a dimension independent $l_2$ norm. 
In this case, we show ridgeless SGD has been sufficient to obtain dimension independent generalization error (Theorem \ref{thm:Asgd}). 
In the second scenario, the true parameter is dimension independent only under some problem-specific norms. 
In this more challenging case, we show SGD can achieve dimension independent generalization error by proper amount of ridge penalty (Theorem \ref{thm:sgd}).}
The upper bound of the generalization error is provided. Using linear regression as an illustrative example,  we show that  each term in the generalization bound has a strong statistical interpretation (see Section \ref{sec:SA}). The upper bound also leads to practical guidelines on the rates of  problem-related parameters, which are given by Corollary \ref{cor:selectpara_online}. 


While Theorem \ref{thm:sgd} provides an  error bound on generalization, for this error bound to be almost dimension-independent, we require the \emph{effective dimension} to be small. Section \ref{sec:nooverfit} first formally defines the general notion of \emph{low effective dimension}, which can essentially be described by 1) the loss function has a fast decaying Hessian spectrum, and  2) the true parameter is either sparse or uniformly bounded along Hessian's eigen-directions.

In Section \ref{sec:linear}, we carefully investigate  
generalization errors in various linear models. We consider the cases of finite projections of infinite-dimensional models and linear regression with redundant features. 
In these scenarios, we quantify when the overparameterization does not hurt the generalization performance.

Our generalization result can also be applied to a wide range of nonlinear models. In Section \ref{sec:non_linear}, we study both convex nonlinear models such as logistic regression and non-convex models such as $M$-estimator with the Tukey's biweight loss function \citep{Tukey:60} and two-layer neural networks. We show that the low effective dimension naturally occurs in these applications. 


\subsection{Related works}
High dimensional statistical learning with low effective dimension is a common idea in principal component analysis (PCA) and  functional data analysis applications \citep{Jol02,RS05}.  To implement PCA,  we first find a $k^*$-dimensional subspace  of which the truncated variability  contains a significant proportion (e.g., $95\%$)  of the full model variability. Then we project the problem onto this subspace and conduct the data analysis on it. In general, the truncated dimension $k^*$ needs to be less than the sample size $N$, and the statistical problem needs to  be well-conditioned in the $k^*$-dimensional subspace. {While we also require the existence of a low effective dimension}, the problem setting considered in this paper is fundamentally different from the classical PCA problem. In our setting, the SGD algorithm directly runs on the original problem space of dimension $p$,  and there is no projection step. Moreover, we do not need the problem to be well-conditioned in the ambient space.
We would argue that our setting is more useful for practical applications since it is difficult to find a proper truncation dimension $k^*$ and PCA is intractable with rank-deficient data. 


In recent years, the generalization error bounds for linear models in overparameterized settings have been carefully investigated in the literature, see, e.g.,  \cite{BLLT:19, Nakkiran19, AKT:19, Hastie19}.  Our result is different from these existing results in the following perspectives:
\begin{enumerate}[1)]
	\item  Nonlinearity: Most of these works focus on linear models. For example, \cite{BLLT:19} defined two notions of \emph{effective ranks}  of the data covariance matrix in linear models and expressed the generalization error bound in terms of effective ranks.  A more detailed comparison between our low effective rank conditions and the ones in \cite{BLLT:19}  will be provided in Remark \ref{rem:comparewithBLLT}. 	In addition, \cite{Hastie19} developed the generalization error bound for the composition of an activation function and a linear model. Our work also requires conditions similar to small effective ranks, but they are simpler in formulation, as discussed in Section \ref{sec:nooverfit}. Moreover, our results can be applied to  general nonlinear  models. 
	\item Anisotropic spectrum and regularization: \cite{Nakkiran19} and \cite{Hastie19} focus their studies on cases with isotropic or well-conditioned regressor covariance matrices. Our study focuses on cases with anisotropic regressor covariance, of which the minimum eigenvalue decays to zero. 
	Moreover, since \blue{online learning has the implicit regularization effect }and the ridge penalty is applied in certain scenarios, we do not have the ``double descent'' phenomenon as in other literature \citep{BHM:18,BHMM:19}. A recent paper by \cite{NVKM:20} also showed that for certain linear regression
	models with isotropic data distribution, the ridge penalty regularized regression (in the offline setting) can avoid the ``double descent'' phenomenon.
	
	\item Online optimization: The aforementioned works mainly focus on offline optimization. For example, \cite{BLLT:19}, \cite{Nakkiran19}, and \cite{Hastie19} show that the generalization error of the offline linear regression solution is closely related to the spectrum of the design matrix. 
	In particular, when the minimum eigenvalue is close to zero, the offline learning results become unstable because of singular matrix-inversions. 
	Since the design matrix relies on data realization, such instability can only be studied through random matrix theories (RMT). 
	While these studies of offline linear regression are interesting and technically deep, their dependence on RMT makes the extensions to nonlinear models difficult. 
	In comparison, online learning methods process one data point at a time and do not involve the inversion of design matrices, which facilitates the study of general nonlinear models.   
	Moreover, in terms of practical applications, online optimization methods such as SGD are  appealing due to their low per-iteration complexity compared to offline optimization.  Therefore, this paper focuses on the generalization error for  online learning in overparameterized settings. 
	
\end{enumerate}

For online optimization, the stochastic gradient descent (SGD), which dates back to \cite{robbins1951}, is perhaps the most widely used method in practice. The convergence rates of the SGD for convex loss functions have been well studied in the literature (see, e.g., \citealp{zhang2004solving}, \citealp{nesterov2008confidence}, \citealp{BM11},  \citealp{Lan12OSGD}, \citealp{Bach:12}, \citealp{bach2013non}). For the constant stepsize SGD,  \cite{BM11} provided the generalization error bound in Theorem 1 of their paper. However, their bound has an explicit linear dependence on $p$, which is not applicable to the overparameterized setting. Our paper provides a more refined generalization error bound, which incorporates the Hessian spectrum to capture the ``low dimension effect'' in the overparameterized setting.
\blue{  When ridge regularization is presented, existing results show that the generalization error is dimension independent \citep{shalev2014understanding, zhang2005learning,carratino2018learning}. But, in general, these results need the norm of population gradient $\nabla F$ to be dimension independent (See Section 14.5.3 in  \cite{shalev2014understanding}). Our new results do not have this restriction but only require  stochastic gradient  variance to be independent of the dimension. This is a much less stringent  restriction because the variance can be reduced by various techniques (see, e.g., \cite{johnson2013accelerating}).}

Overparameterized neural network (NN) is a very active research direction. 
There are several existing works explaining why overfitting does not happen in large NN   \citep{GRS:18,LL2018,NLBLS:19,ALL:19,ADHLW:19, EMW19, EMWGD19}. 
Interestingly, the conditions they impose are largely similar to  the ones we will use.
Namely, they require the high dimensional input data and the Frobenius norm of true weight matrices to be  bounded by constants. For this to be true, only a small portion of the data or model components can be significantly active, which satisfies  the concept of low effective dimension. 

On the other hand, our study of two-layer NN in Section \ref{sec:NN} is different from existing results in the following perspectives.  One popular way to analyze generalization error is by analyzing the Rademacher complexity \citep{GRS:18, NLBLS:19, EMW19}, which can be used to establish an upper bound of the difference between training error and generalization error  (see, e.g.,  Theorem 3.3 of \cite{MRT18}). However, these results are usually derived in an offline optimization setting, while our paper focuses on the result from online SGD.  \cite{LL2018} and \cite{ALL:19} both studied NN generalization error with SGD iterations. But they mainly focused on classification scenarios where the loss function is bounded. Moreover, \cite{LL2018} required the loss function to be of a logistic form, and \cite{ALL:19} studied the running average generalization error. Our results can be applied to regression NN with unbounded loss functions. \cite{ADHLW:19} and  \cite{EMWGD19} studied NN generalization bound with deterministic gradient descent. Moreover, their studies assume a certain data angle or Gram matrix to have a strictly positive minimal eigenvalue. Please see Section \ref{sec:NN} for more detailed comparisons with these works.



\section{A Generalization Bound}
\label{sec:generalize}
In this section, we present a general result on the generalization bound for the SGD solution from \eqref{eq:sgd_iter}.

\subsection{Preliminaries: high dimensional norms and non-convex energy landscape}
\label{sec:prelim}
\blue{ The main issue this paper tries to understand is the high dimensional generalization error when using SGD as the training method. 
In some simple scenarios, the true model parameters are ``sparse", so $\|w^*\|$ does not grow with the dimension. 
This allows us to measure the distance between SGD iterate $w_n$ and $w^*$ directly. 
In other scenarios, $w^*$ can be very dense, and $\|w^*\|$ may grow linearly or even faster with the dimension. 
To resolve this issue, we define  the following norms: }
\begin{defn}\label{defn}
	Given a matrix $A\in \reals^{p\times p}$ and $\lambda$, we decompose $\reals^p=S_\lambda\oplus S_\bot$, where $S_\lambda$ consists of eigenvectors of $A$ with eigenvalues above $\lambda >0$ and $S_\bot$ is the orthogonal complement of $S_\lambda$. Given any vector $v$, denote its decomposition as $v=v_\lambda+v_\bot$, where $v_\lambda\in S_\lambda$ and $v_\bot\in S_\bot$. Then define
	\begin{align}\label{def:newnorm}
	\|v\|^2_{A}=v^TA v,\quad \|v\|^2_{A,\lambda}:=\lambda\|v_\lambda\|^2+ v_\bot^TAv_\bot.
	\end{align}
\end{defn}
 We introduce the norm $\|v\|^2_A$, whose value can be independent of the ambient dimension $p$.   The second norm $\|v\|^2_{A,\lambda}$ is a truncated version of the first norm, it essentially truncates all eigenvalues of $A$ above $\lambda$ to $\lambda$.
It is easy to see that $\|v\|^2_{A,\lambda}\leq \|v\|^2_A$. We introduce the second norm  because $\|v\|^2_{A,\lambda}$ converges to zero when the regularization parameter $\lambda$ does, while $\|v\|^2_A$ is independent of $\lambda$. 

\blue{ It is well known that SGD works well for convex problems. This is also true for our theoretical analysis, which works best in convex settings.
On the other hand, many practical problems are non-convex,  
which makes the statistical learning problem technically more challenging.}
 First, the function $F$ can have multiple local minima, and each local minimum has an attraction basin, which is a ``valley" in the graph of $F$. 
Within each valley, we assume that $F$ is \blue{ locally convex. 
Machine learning and theoretical deep learning literature often study the scenarios that local minima lie in large and shallow valleys, which lead to more stable generalization performance.}
\blue{Suppose $\calD$ is the attraction basin of the optimal solution $w^*$ in \eqref{eq:sto_opt}, initializing SGD in $\calD$ will generate iterates converging to $w^*$  with high probability.  So a natural question is the gap between the estimator learned from SGD and $w^*$.}
On the other hand, if the SGD iterates  take place outside $\calD$, the output can be irrelevant to the properties of $w^*$. 
\blue{ Therefore, we need to introduce a stopping time $\tau$, which describes the first time SGD exits $\calD$:
\[
\tau=\min\{n: w_n\notin \calD\},
\]
where $w_n$ is the $n$-th SGD iterate defined in \eqref{eq:sgd_iter}. Our generalization analysis will assume the SGD iterates stay within $\calD$. We will also provide bounds of probability that SGD leaves $\calD$. }

\blue{While it is reasonable to assume that $w^*$ as a local minimum is in the valley $\calD$, verifying this assumption can be difficult. For example, because we do not have access to the population loss function $F$ or the Hessian directly, to check the convexity of $F$, we need to investigate $\Fhat$ instead. There will be a certain inaccuracy due to the randomness in $\Fhat$.  As another example, while we know the Hessian of $F$ is positive semidefinite at $w^*$ because $w^*$ is a local minimum, $F$ does not have to be convex in the neighborhood  $w^*$. For both these examples, it is more accurate to say $F$ is ``approximately convex" in $\calD$. Our analysis can also extend to such very challenging cases by introducing proper regularizations.  On the other hand,  if $F$ is very non-convex in $\calD$, then one should not expect that SGD will produce good learning results. Therefore, to achieve a reasonable generalization error, we need $F$ to be very close to convex.  In many applications, this can be done by choosing either a very large sample size $N$ or a very small neighborhood around $w^*$. }

Based on our discussion, we have the following assumption on the population risk function,
\begin{aspt}
	\label{aspt:convex}
	The optimal solution $w^*$ of the population risk $F$ has a neighborhood $\calD$, such that for some positive semidefinite (PSD) matrix $A$ and $\delta \in [0,1/2)$, 
	\begin{equation}\label{eq:cond}
	-\delta A\preceq  \nabla^2F(w) \preceq A,\quad \forall \; w\in \calD. 
	\end{equation}
\end{aspt} 
In Assumption \ref{aspt:convex} and in the sequel, for two symmetric matrices $C, D$, $C\preceq D$ indicates that $D-C$ is positive semidefinite (PSD). The upper bound on the Hessian matrix $\nabla^2F(w) \preceq A$ is widely assumed in the statistical literature. 
The parameter $\delta$ above describes the level of non-convexity. In particular, $\delta=0$ indicates that $F$ is convex within $\calD$.
However, our condition in \eqref{eq:cond} is more general since $\delta$ can be strictly positive, which allows $F$ to be non-convex.
On the other hand, although our generalization error bound holds for $\delta\in [0,1/2)$, for this upper bound to be smaller than a certain threshold, $\delta$ needs to be small.  Please see Corollary \ref{cor:selectpara_online} for the exact dependence of $\delta$ in the upper bound.

\subsection{\blue{Generalization bound with sparse true parameter }}
\label{sec:main}
\blue{ We will discuss two possible scenarios with high dimensional machine learning. In the first case, the true parameter $w^*$ is sparse or weak so that its $l_2$ norm does not depend on the dimension. In this case, we have the following results: }
%
	\blue{
\begin{thm}
	\label{thm:Asgd}
	Under Assumption \ref{aspt:convex}, suppose $w_0\in \calD$ and  there is a constants $r$ and $c_r$  such that
	\begin{equation}
	\label{eqn:sgvarianceless}
	\E \| \nabla f(w,\zeta)-\nabla F(w)\|^2\leq  r^2+c_r |(w-w^*)^T\nabla F(w)|,\quad \forall w\in \calD.
	\end{equation}
	Then if the ridgeless SGD stepsize $\eta$ and the regularization parameter  $\lambda$ satisfy
	\begin{align*}
	\eta \leq\min\Big\{ \frac1{4(1+c_r)(\|A\|+\lambda)}, 1\Big\}, \quad 2\delta \|A\| \leq \lambda \leq 1, 
	\end{align*}
we have the generalization error for the averaged SGD,
	\begin{align}
	\label{eqn:opt2ASGD}
	&\E \left[1_{\tau\geq N-1} G(\bar{w}_N)\right]\leq \frac{2\E \|w_0-w^*\|^2}{N\eta}+2\eta r^2+8\lambda\|w^*\|^2,
	\end{align}
	where the generalization error $G$ is defined in \eqref{eq:generalization} and $\bar{w}_N=\frac1N\sum_{i=0}^{N-1} w_i $ is the averaged SGD iterate.
\end{thm}
We note that when $F$ is locally convex in $\mathcal{D}$ (i.e., $\delta=0$ in \eqref{eq:cond}), this result incorporates the implicit regularization case by setting $\lambda=0$. The generalization error bound in \eqref{eqn:opt2ASGD} contains three terms. The first term decays with the sample size $N$. The second term is controlled by the stepsize $\eta$. 
The last term $8\lambda\|w^*\|^2$ is a bias caused by the regularization. Since $F_\lambda$ is different from $F$, the minimizer of   $F_\lambda$ is also different from $w^*$.
In other words, if $F$ is convex in $\calD$ (i.e., $\delta=0$), it is actually better to do rigdeless regression by setting $\lambda=0$. Note that in this case of implicit regularization, the generalization error bound in \eqref{eqn:opt2ASGD} is dimensional independent when  $\|w^*\|^2$ is dimension independent with the initialization  $w_0=\mathbf{0}$.
}

\subsection{Generalization bound with non-sparse true parameter}
\blue{ In some scenarios, $\|w^*\|$ may grow with the dimension.  While the results in Theorem \ref{thm:Asgd} still hold, the estimate \eqref{eqn:opt2ASGD} is no longer dimension independent. In this case, the learning cannot rely only on implicit regularization. We will need to introduce regularization and high dimensional norms as in Definition \ref{defn} for these problems: }
\begin{thm}
	\label{thm:sgd}
	Under Assumption \ref{aspt:convex}, suppose $w_0\in \calD$ and  there is a constants $r$ and $c_r$  such that
	\begin{equation}
	\label{eqn:sgvariance}
	\E \| \nabla f(w,\zeta)-\nabla F(w)\|^2\leq  r^2+c_r r^2\min\{G(w),\|w\|^2\},\quad \forall w\in \calD.
	\end{equation}
	Then if the SGD hyper-parameters, the stepsize $\eta$ and the regularization parameter $\lambda$, satisfy
	\begin{align*}
	\eta \leq \min \Big\{1, \dfrac{\lambda}{12\|A\|^2+6 \lambda^2+6c_r r^2}, \dfrac{1}{12\|A\|}, \frac{\lambda}{6c_r\|A\|r^2} \Big\},
	\quad 4\delta\|A\|\leq \lambda\leq 1,
	\end{align*}
\blue{	we have 
	\begin{align}\label{eqn:opt2}
	\E [ G(w_{N})1_{\tau\geq N}] &\leq  4 \|w^*\|^2_{A,\lambda}+\frac{C_1}{\lambda} (\eta + \delta) \\
	\notag
	&\quad +\exp(-\tfrac14\lambda N \eta  )\E [ G(w_0)+ 4N\|A\|\|w_0\|^2 ],
	\end{align}
	with $C_1=60\|A\| \left(r^2+\|w^*\|_A^2\right) + 10\|w^*\|_A^2$. }
\end{thm}
\blue{ While it is not completely new that SGD on convex ridge regression has dimension independent generalization error \citep{shalev2014understanding, zhang2005learning,carratino2018learning}, in general, these results need the norm of gradient $\nabla F$ to be dimension independent (See Section 14.5.3 in \cite{shalev2014understanding}). Theorem \ref{thm:sgd} does not have this restriction. In fact, the population gradient $\nabla F$ can be unbounded in many applications (e.g., linear regression).  In contrast, our assumption for the stochastic gradient is imposed on its variance, see 
\eqref{eqn:sgvariance}. This is a much relaxed assumption because the variance can be reduced by various techniques (e.g., \cite{johnson2013accelerating}) or simply by increasing the mini-batch size for stochastic gradient computation.  }

\subsection{\blue{SGD configuration with given  generalization error target }}
In Section \ref{sec:nooverfit}, we will explicitly define the low effective dimension so that $C_1$ and the upper bound in \eqref{eqn:opt2}  are independent of dimension $p$, or depend on $p$ only via a polynomial logarithmic factor.
This differentiates our result from the estimates in existing literature on SGD, e.g., \cite{BM11}. In particular, most existing results do not take the Hessian spectrum into consideration, so they will have linear or stronger dependence on the dimension $p$; while our results do not necessarily have such dependence. 


In the upper bound \eqref{eqn:opt2}, each term  carries a strong statistical interpretation, which will be illustrated using linear models in the next subsection (see Section \ref{sec:SA}). 
In particular, the term $\|w\|^2_{A,\lambda}$ in \eqref{eqn:opt2}  
can be interpreted as the bias caused by minimizing $F_\lambda$ instead of $F$, it decays with the regularization parameter $\lambda$ shrinking to zero. The term $\frac{C_1\eta}{\lambda}$ is the variance induced by the SGD algorithm, which increases as $\lambda$ decreases. This reveals that under our current problem setting, $\lambda$ controls a bias-variance tradeoff. Ideally, we can choose  small $\lambda$ and stepsize $\eta$ to make both the bias and variance small. However, this comes with a price. As the convergence rate scales with $\lambda\eta$, so using small $\lambda$ and $\eta$ need to be compensated with a large sample size $N$ (i.e., the number of iterations in SGD).


We can quantify these tradeoffs by considering a practical scenario where a generalization error is pre-fixed to be $\epsilon$, then our results provide guidelines on how to tune the parameters:
\blue{
\begin{cor}[Corollary of Theorem \ref{thm:Asgd}]\label{cor:selectpara_onlineridgeless}
	Suppose there is an universal constant $C_0$ such that 
	\[
	\| w^*\|^2, \| w_0-w^*\|^2 \leq C_0.
	\]
	Given any $\epsilon>0$, if the regularization parameter $\lambda(\epsilon)$, the stepsize $\eta(\epsilon)$, the non-convexity parameter $\delta(\epsilon)$, and the sample size  $N(\epsilon)$ satisfy 
	\begin{align}
	\label{eq:parameter2.5}
	&\lambda(\epsilon) \leq \min \left\{ \frac{\epsilon}{8C_0},1\right\},\quad  \delta(\epsilon) \leq \min \Big\{ \frac{\epsilon}{16C_0\|A\|}, \frac{1}{2\|A\|}\Big\},  \\
	\notag
	& \eta(\epsilon) <\min\left\{ \frac1{2(1+c_r)(\|A\|+\lambda(\epsilon))},\frac{\epsilon}{2r^2},1\right\},\quad N(\epsilon) > \frac{2C_0}{\epsilon\eta(\epsilon)},
	\end{align}
	then $\E [ G(\bar{w}_N)1_{\tau\geq N} ] \leq 3\epsilon$. \end{cor}
}

\begin{cor}[Corollary of Theorem \ref{thm:sgd}]\label{cor:selectpara_online}
	Given any $\epsilon>0$, if the regularization parameter $\lambda(\epsilon)$, the stepsize $\eta(\epsilon)$, the non-convexity parameter $\delta(\epsilon)$, and the sample size  $N(\epsilon)$ satisfy 
	\begin{align}
	\label{eq:parameter}
	&4 \| w^*\|_{A,\lambda(\epsilon)}^2< \epsilon, \quad \delta(\epsilon) \leq \frac{\lambda(\epsilon)\epsilon}{C_1} , \quad \eta(\epsilon) < \frac{\lambda(\epsilon)\epsilon}{C_1},  \\
	\notag
	& N(\epsilon) > \max\left\{ \frac{-4\log\{\epsilon/2\E [ G(w_0) ] \}}{\lambda(\epsilon) \eta(\epsilon)}, \frac{-8\log\{\epsilon\lambda(\epsilon)\eta(\epsilon)/(64\|A\|\E [ \|w_0\|^2 ]) \}}{\lambda(\epsilon) \eta(\epsilon)}  \right\},
	\end{align}
	and the conditions of Theorem \ref{thm:sgd} hold, then $\E [ G(w_N)1_{\tau\geq N} ] \leq 4\epsilon$. 
\end{cor}

It is noteworthy that \eqref{eqn:opt2} only discusses the scenario where SGD iterates stay in the domain $\calD$. 
This is necessary since all our conditions are imposed only within $\calD$. Once an SGD iterate leaves $\calD$, there is no particular reason it can get back to  $\calD$.
Another possible improvement is to find an upper bound for the conditional generalization error $\E[G(w_N)|\tau\geq N]$. But this is not feasible when $\calD$ is a general region. For example, if $\calD$ is the intersection between any set and $\{w:G(w)\geq G(w_0)-1\}$, the conditional expectation will be larger than $G(w_0)-1$, while $\{\tau\geq N\}$ may have a nonzero occurrence probability. 
 In Section \ref{sec:conclusion}, we will briefly discuss how to remedy such a restriction in practice through data splitting.

\blue{
In practice, SGD is often implemented with mini-batch data to reduce the noise within stochastic gradient. With regularization, the mini-batch SGD with batch-size $J$ can be formally written as 
\[
w_{n+1}: =w_{n} -\frac{1}{J}\eta \sum_{k=Jn+1}^{J(n+1)}(\nabla f(w_{n}, \zeta_k)+\lambda w_n),
\]
where \emph{i.i.d.} data $\{\zeta_{Jn+1},\ldots, \zeta_{J(n+1)}\}$ forms the $n$-th batch of data. Our results can be extended to  mini-batch SGD as well. To see this, we simply let 
$z_n=\{\zeta_{Jn+1},\ldots, \zeta_{J(n+1)}\}$ and 
\[
\tilde{f}(w,z_n)=\frac{1}{J} \sum_{k=Jn+1}^{J(n+1)}f(w_{n}, \zeta_k).
\]
Note that $\E \tilde{f}(w,z_n)=\E f(w,\zeta_i)=F(w)$. It is also straight forward to see that applying our SGD formulation \eqref{eq:sgd_iter} on $\tilde{f}$ with $z_n$ leads to the mini-batch SGD. Applying our results, e.g. Corollaries \ref{cor:selectpara_onlineridgeless} and \ref{cor:selectpara_online}, to mini-batch  SGD requires simple modifications for only two parameters. First, the variance of stochastic gradient $\nabla \tilde{f}(w,z)$ is only $\frac{1}{J}$ of the variance of $\nabla f(w,\zeta)$,  so the parameter $r^2$ in mini-batch SGD should be $\frac{1}{J}$ of $r^2$ in the standard SGD. Second, because each iteration of the mini-batch SGD requires $J$ data samples, so the overall sample size should be $NJ$.
}

\subsection{Statistical interpretation in linear models and bias-variance tradeoff}
\label{sec:SA}

To facilitate better understanding our results, we will use linear models to illustrate the statistical interpretation of each term in the generalization upper bound in \eqref{eqn:opt2}.
%
In linear regression, each \emph{i.i.d.} observation contains a pair of dependent and response  variables, $\zeta_i = (x_i, y_i)\in \reals^p\times \reals$, where $y_i$ is generated by the following model 
\begin{equation}\label{eq:linear_model0}
y_i = x_i^{T} w^* + \xi_i.
\end{equation}
In \eqref{eq:linear_model0},  $w^*$ is the true regression coefficient, and  the noise term $\xi_i$ is independent of $x_i$ with zero mean and a finite variance $\sigma^2$. For the ease of illustration, we assume $x_i \sim \mathcal{N}(0, \Sigma)$. The generalization error of this linear model takes the following form,
\begin{align}
\label{eqn:linloss1}
G(w) =\E\Big[ \frac12(y_i-w^T x_i)^2-\frac12(y_i-(w^*)^T x_i)^2 \Big]= \frac{1}{2} (w-w^*)^T\Sigma(w-w^*).
\end{align}
From \eqref{eqn:linloss1}, we can see that $G(w)$ has a strong dependence on the structure of $\Sigma$. Let us denote the eigenvalues of $\Sigma$ by  $\lambda_1\geq \lambda_2\geq\cdots\geq \lambda_p$, and the eigenvector corresponding to $\lambda_i$ by $v_i$. Then the parameter error, $v^T_i(w-w^*)$, contributes to $G(w)$ via the factor $\lambda_i$.

It is well known that SGD can be interpreted as a stochastic approximation of the gradient descent \citep{robbins1951}, namely we can rewrite \eqref{eq:sgd_iter} as
\begin{equation}\label{eq:sgd_stoc_approx}
w_{n+1}=w_n-\eta\nabla F_\lambda(w_n)+\eta \xi_n,
\end{equation}
where $\xi_n=-\nabla f(w_n,\zeta_n)+\nabla F(w_n)$ is the noise in stochastic gradient. For the quadratic loss with the ridge penalty  $F_\lambda$, the SGD iterates in \eqref{eq:sgd_stoc_approx} take the following form,
\begin{equation}
\label{eqn:SA}
w_{n+1}=w_n-\eta\Sigma (w_n-w^*)-\eta \lambda w_n+\eta \xi_n=(I-(\Sigma+\lambda I)\eta)(w_n-w^*_\lambda)+w^*_\lambda+\eta \xi_n,
\end{equation}
where $w^*_\lambda:=(\Sigma+\lambda I)^{-1}\Sigma w^*$ is the minimizer of $F_\lambda(w)=\E \frac{1}{2} (y- w^T x)^2 + \frac{\lambda }{2} \|w\|^2$. 

It is easy to see that $w_n$ then follows a vector autoregressive (VAR) model (see, e.g.,  Chapter 8 of \citealp{Tsay}). For the ease of discussion,  we simply treat $\xi_n$ as $\mathcal{N}(0,\frac{r^2}{p}I)$, so $\E \|\xi_n\|^2=r^2$ as we assumed in Theorem \ref{thm:sgd}. 
Admittedly, the independent Gaussian assumption here contradicts our anisotropic assumption. We are assuming it here to gain an intuitive understanding.
Then the stationary distribution of $w_n$ in \eqref{eqn:SA} is a Gaussian $\mathcal{N}(\mu,V)$. Take expectation and covariance on both sides of  \eqref{eqn:SA}, and the $\mu$ and $V$ take the following form (see Chapter 8.2.2 of \citealp{Tsay} ),
\[
\mu=(I-(\Sigma+\lambda I)\eta)(\mu-w^*_\lambda)+w^*_\lambda\Rightarrow \mu=w^*_\lambda,
\]
\[
V=(I-(\Sigma+\lambda I)\eta)V(I-(\Sigma+\lambda I)\eta)+\frac{\eta^2 r^2}{p} I.
\]
When the stepsize $\eta\leq \|\Sigma+\lambda I\|^{-1}$, $(\Sigma+\lambda I)^2\eta\preceq \Sigma+\lambda I $, so
\[
V=\frac{r^2\eta}{p}(2(\Sigma+\lambda I)-(\Sigma+\lambda I)^2\eta)^{-1}\preceq \frac{r^2\eta}{p}(\Sigma+\lambda I)^{-1}.
\]
These results give us the limiting average generalization error 
\begin{align}
\notag
\lim_{n\to \infty}\E G(w_n)&=\frac{1}{2}(w^*_\lambda-w^*)^T \Sigma(w^*_\lambda-w^*)+\frac12\tr(V\Sigma)\\
\label{eqn:biasvar}
&\leq G(w^*_\lambda)+\frac{r^2\eta}{2p}\tr((\Sigma+\lambda I)^{-1}\Sigma). 
\end{align}
The first term $G(w^*_\lambda)$ is the bias caused by using regularization. Indeed, the optimizer of $F_\lambda$ is $w^*_\lambda$ rather than  $w^*$. Recall that $(\lambda_i, v_i)$ are the eigenvalues and eigenvectors of $\Sigma$. We define $a_i=\langle v_i, w^*\rangle$ and further express $G(w^*_\lambda)$  in \eqref{eqn:biasvar} as follows,
\begin{align*}
G(w^*_\lambda)&=\frac12 ((\Sigma+\lambda I)^{-1}\Sigma w^*-w^*)^T\Sigma((\Sigma+\lambda I)^{-1}\Sigma w^*-w^*)\\
&=\frac12 \lambda^2(w^*)^T(\Sigma+\lambda I)^{-1}\Sigma (\Sigma+\lambda I)^{-1}w^*=\frac12\sum_{i=1}^p  \frac{\lambda^2 \lambda_i a_i^2}{(\lambda+\lambda_i)^2}.
\end{align*}
Note that when $\lambda_i\geq 0$, $\frac{\lambda^2\lambda_i}{(\lambda_i+\lambda)^2}\leq \frac{\lambda^2\lambda_i}{\lambda^2}=\lambda_i$, and by Young's inequality $\frac{\lambda^2\lambda_i}{(\lambda_i+\lambda)^2}\leq \frac{\lambda^2\lambda_i}{4\lambda_i\lambda}\leq \lambda$. Therefore, we have the following upper bound of $G(w^*_\lambda)$
\begin{equation}\label{eq:G_upper}
G(w^*_\lambda)\leq \frac12\sum_{i=1}^p (\lambda\wedge \lambda_i) a_i^2=\frac12\|w^*\|_{\Sigma,\lambda}. 
\end{equation}
The upper bound in \eqref{eq:G_upper} is essentially the first term in \eqref{eqn:opt2} by noticing that $\Sigma = A = \nabla^2 F(w)$ in linear regression, which  gives an upper bound for the bias.


For the second variance term in \eqref{eqn:biasvar}, 
\begin{equation}\label{eq:var}
\text{var}(\lambda):= \frac{r^2\eta}{2p}\tr((\Sigma+\lambda I)^{-1}\Sigma)=\frac{r^2\eta}{2p}\sum_{i=1}^p \frac{\lambda_i}{\lambda_i+\lambda}\leq\frac{r^2\eta}{2p}\sum_{i=1}^p \frac{\lambda_i}{\lambda}=\frac{\eta r^2\lambda_1}{2\lambda}.
\end{equation}
This upper bound is essentially the second term in the generalization upper bound in \eqref{eqn:opt2}, as it  depends linearly on $\eta, \lambda^{-1}, \lambda_1 r^2$. 

The first two terms in  \eqref{eqn:opt2} are based on the limiting average generalization error. With finite SGD iterations, the iterate $w_n$ may not reach the limiting distribution. On the other hand, for VAR models, it is well known that the speed of convergence for $w_n$ is exponential, and the convergence rate is closely related to  the minimum eigenvalue  $\lambda_{\min}((\Sigma+\lambda I)\eta)=\lambda\eta$  
(see, e.g., \citealp{Tsay} Chapter 8.2.2). The finite iterate error leads to the third term of $\exp(-\frac14\lambda \eta N)$ in generalization error bound in \eqref{eqn:opt2}.

\blue{ In the special case that $\lambda=0$, \eqref{eqn:biasvar} reduces to 
\[
\lim_{n\to \infty}\E G(w_n)\leq G(w^*)+\frac{r^2\eta}{2p}\tr(I)=r^2\eta.
\]
}

Finally, we consider the scenario where $\Sigma$ is indefinite with $\delta=-\lambda_{\min}(\Sigma)>0$. While the population loss $F$ is non-convex, when adopting $\lambda>2\delta$, we have that $\Sigma+\lambda I$ is positive definite and $F_\lambda$ is convex.
Then the generalization upper bounds need to be updated by replacing $\lambda$ with $\lambda-\delta$, which leads to a perturbation on the  order of $\delta$. In particular, note that by Young's inequality, the derivative of the bias term with respect to $\lambda$ is bounded by
\[
|\partial_\lambda G(w^*_\lambda)|=\sum_{i=1}^p  \frac{\lambda \lambda^2_i a_i^2}{(\lambda+\lambda_i)^3}\leq 
\sum_{i=1}^p  \frac{\lambda_i a_i^2}{4(\lambda+\lambda_i)}\leq \frac{1}{4\lambda}\sum_{i=1}^p \lambda_ia_i^2=\frac{\|w^*\|^2_\Sigma}{4\lambda}.
\]
The derivative of the variance term with respect to $\lambda$ in \eqref{eq:var} is bounded by Young's inequality,
\[
|\partial_\lambda \text{var}(\lambda)|=\frac{r^2\eta}{2p}\sum_{i=1}^p \frac{\lambda_i}{(\lambda_i+\lambda)^2}\leq 
\frac{r^2\eta}{2p}\sum_{i=1}^p \frac{\lambda_1}{\lambda^2}\leq \frac{r^2 \eta\lambda_1}{2\lambda^2}. 
\]
Therefore, replacing $\lambda$ with $\lambda-\delta$ to handle non-convexity,  we need to add the following term in the generalization error bound,
\[
\delta |\partial_\lambda F(w^*_\lambda)|+\delta |\partial_\lambda \text{var}(\lambda)|\leq \delta\left(\frac{\|w^*\|^2_\Sigma}{4\lambda}+\frac{r^2\lambda_1 \eta}{2\lambda^2}\right).
\]
This term can be further upper bounded by the last term of \eqref{eqn:opt2}.


%
%
%

\section{Low Effective Dimension}
\label{sec:nooverfit}

Given the generalization error bound in Theorem \ref{thm:sgd}, we introduce the concept of ``low effective dimension'' and show that the generalization error bound in \eqref{eqn:opt2} can be independent (or dependent poly-logarithmically) of the ambient dimension $p$ in an overparameterized regime. We will use the $O$ and $\Omega$ notations to hide constants independent of $p$ and use the $\tildeO$ and $\widetilde{\Omega}$ notations to hide constants depend poly-logarithmically on $p$.
In particular, we introduce the following standard asymptotic notations: $A_\epsilon=O(f(\epsilon)), B_\epsilon=\tildeO(f(\epsilon)), C_\epsilon=\Omega(f(\epsilon)), D_\epsilon=\widetilde{\Omega}(f(\epsilon))$. These notations mean that there exist some universal constants $c$ and $C>0$ such that, 
\[
A_\epsilon\leq Cf(\epsilon),\quad B_\epsilon\leq C(\log p)^{c} f(\epsilon),\quad  C_\epsilon\geq Cf(\epsilon),\quad D_\epsilon\geq C(\log p)^c f(\epsilon). 
\]

\subsection{Initialization and stochastic gradient variance}
We investigate the terms that appear in the generalization bound \eqref{eqn:opt2}: whether they can be independent of $p$; and how they affect the necessary sample size $N(\epsilon)$ in \eqref{eq:parameter}. 

First, we notice that the terms related to initialization $w_0$, i.e., $\E \|w_0\|^2$ and $\E G(w_0)$,  appear in the sample size $N(\epsilon)$ in \eqref{eq:parameter}. If the region $\calD=\reals^p$, we can often choose appropriate $w_0$ so that $\E \|w_0\|^2$ and $\E G(w_0)$ are independent of $p$. For example, for linear regression loss function in \eqref{eqn:linloss1}, we can pick $w_0=0$, then $\E G(w_0)=\frac12 \|w^*\|^2_A$ with $A=\Sigma$, which will be bounded by an $O(1)$ constant as shown below.  For a restrictive region $\calD$,  although $\E \|w_0\|^2$ and $\E G(w_0)$ may scale as a polynomial function of $p$,  $N(\epsilon)$ only depends logarithmically on these two terms. Therefore,  the dimension dependence of $N(\epsilon)$ is only logarithmic.

Second, we consider the stochastic gradient variance $r^2$, which contributes to the term $C_1$ in \eqref{eqn:opt2}. In a typical setting, it scales roughly as the squared population gradient, i.e., 
\begin{align*}
\E \|\nabla f(w,\zeta)-\nabla F(w)\|^2&\approx O(\E \|\nabla F(w)\|^2)\\
&=O(\E \|\nabla F(w)-\nabla(F(w^*))\|^2)\\
&=O(\E \|\nabla^2F(w) (w-w^*)\|^2)\\
&\mbox{Assume that $w\sim \mathcal{N}(0,I_p)$ and $\nabla^2 F\preceq A$}\\
&=O\big(\|A\|(\|w^*\|_A^2+\tr(A))\big). 
\end{align*}
We will see such an approximation holds for many applications of interest. Moreover, we have $\|A\|\leq\tr(A)$, which can often be $p$-independent as discussed below. The scale of $\|w^*\|_A$ will also be discussed next.

From the discussion above, we only need to focus on two terms in \eqref{eqn:opt2}, $\|w^*\|_{A}$ and $\|w^*\|_{A,\lambda}$. For the generalization error to be small and independent of $p$, we need to show $\|w^*\|_{A}$ is dimension independent and $\|w^*\|_{A,\lambda}$ decreases as $\lambda$ decreases.

\subsection{Low effective dimension settings}
\label{sec:low_eff}
In this section, we formally define two settings of low effective dimension  as Assumptions \ref{aspt:sparse} and \ref{aspt:trace}. In Sections \ref{sec:linear} and \ref{sec:non_linear}, we will show that these assumptions easily hold for a wide range of convex and non-convex statistical models. 

\subsubsection{Sparse true parameter}

The first setting is characterized in the following assumption.


\begin{aspt}
	\label{aspt:sparse}
	The followings are true
	\begin{enumerate}[1)]
		\item  $\|A\|$ with $A$ defined in Assumption \ref{aspt:convex} is bounded by  an $O(1)$ constant. 
		\item $\|w^*\|$ is  bounded by an $O(1)$ constant.
		\item $r^2,c_r$  defined in  Theorem \ref{thm:sgd} are  bounded by $O(1)$ constants. 
		\item The initial values $\E \|w_0\|^2$ and $\E G(w_0)$  grow polynomially with $p$.
	\end{enumerate} 
\end{aspt}

Assumption \ref{aspt:sparse} can be interpreted as a weak sparsity condition for $w^*$, since there can  be only a few significant components in $w^*$. Sparsity assumption is a very common condition in the statistical literature. However, our assumption only assumes that the $\ell_2$-norm of $w^*$, instead of the $\ell_0$-norm, is bounded. As compared to the $\ell_0$-norm, the $\ell_2$-norm is rotation-free. In addition, we do not need to apply any projection or shrinkage procedures on the SGD iterates.

\blue{ Under Assumption \ref{aspt:sparse},  Corollary \ref{cor:selectpara_online} can be simplified as the following generalization error bound, which shows that the necessary sample size depends on $p$ in a poly-logarithmic factor.  }
\blue{
\begin{prop}\label{prop:sparse}
	Under the conditions in Corollary \ref{cor:selectpara_onlineridgeless} and Assumption \ref{aspt:sparse},  given any $\epsilon>0$, when 
	\begin{equation}\label{eqn:sparsity}
	\lambda(\epsilon) = O(\epsilon ), \ \delta(\epsilon) = O(\epsilon), \ \eta(\epsilon)  = O(\epsilon), \ N(\epsilon) = \Omega\Big(\frac{1}{\epsilon^2}\Big),
	\end{equation}
	we have  $\E [ G(\bar{w}_{N})1_{\tau\geq N-1} ] \leq 3\epsilon$.  	
Moreover, for any $p_0\geq 0$, there are  $A=O(1/p_0)$, if $\calD=\{w: \|w-w^*\|^2\leq A\}$, we have
	\[
	\Prob(\tau\leq N)\leq p_0.
	\]
	Alternatively, if $\delta=0$, for any $\alpha>0$, we take 
	\[
	\lambda(\epsilon) = 0, \ \eta(\epsilon)  = O(\epsilon^{1+\alpha}), \ N(\epsilon) = \Omega\Big(\frac{1}{\epsilon^{2+\alpha}}\Big),
	\]
	we have  $\E [ G(\bar{w}_N)1_{\tau\geq N-1} ] \leq 3\epsilon$. Meanwhile, for any $A>\E\|w_0-w^*\|^2$, if $\calD=\{w: \|w-w^*\|^2\leq A\}$, we have
	\[
	\Prob(\tau\leq N) \leq \frac{\E \|w_{0}-w^*\|^2+  O(\epsilon^{\alpha})}{A}.
	\]
\end{prop}
Proposition \ref{prop:sparse} consists of two parts. The first part shows that the generalization error is of order $O(1/\sqrt{N})$ if the non-convexity is of the same order. The second part demonstrates how to bound the probability of SGD escaping the convexity region $\calD$ if it is a ball centered at $w^*$. For all $\delta\geq 0$, $A$ needs to be of order $1/\delta$ so that the chance of escaping is less than $\delta$. If the problem is convex in $\calD$, $\calD$ just need to include $w_0$ to ensure the chance of no-escaping is nonzero. In both cases, the escape is harder when the radius $\sqrt{A}$ is larger. This also explains why machine learning literature is in favor of local-minima in large valleys. 
}

%
%
\subsubsection{Non-sparse true parameter}
The second setting is technically more interesting, which assumes the data has a low effective dimension in the following sense.  When we say 
a component or a linear combination of components of $w$ is effective,  it means that the loss function $F$ has a significant dependence on it. This can be analyzed through the eigen-decomposition of $\nabla^2 F(w)$ or its upper bound $A$ in \eqref{eq:cond}.   Let $(\lambda_i, v_i)$ be the eigenvalue-eigenvectors of $A$, where $\lambda_i$ are arranged in decreasing order. Then a small $\lambda_i$ indicates that $F$ has a weak dependence along the direction of $v_i$.  For the model to have a low effective dimension, there will be only constantly many $\lambda_i$ being significant, while the remaining eigenvalues in sum have a negligible contribution to the overall loss function.  We formally formulate this setting into the following assumption.
\begin{aspt}
	\label{aspt:trace}
	The followings are true
	\begin{enumerate}[1)]
		\item $\tr(A)$ with $A$ defined in Assumption \ref{aspt:convex} is bounded by an $\tildeO(1)$ constant. 
		\item  the true parameter $w^*$ is bounded in each of $A$'s eigen-direction, in the sense that 
		\begin{equation}\label{eq:W_A_S}
		\|w^*\|_{A,S}:=\max_i\{|\langle v_i,w^*\rangle|, i=1,\ldots,p \}=\tildeO(1).
		\end{equation}
		\item $r^2,c_r$  defined in  Theorem \ref{thm:sgd} are bounded by $\tildeO(1)$ constants.
		\item The initial values $\E \|w_0\|^2$ and $\E G(w_0)$  grow polynomially with $p$. 
	\end{enumerate} 
\end{aspt}

By Cauchy Schwartz inequality,  we have $\|w^*\|_{A,S}\leq \|w^*\|.$ So Assumption \ref{aspt:trace}  condition 2) is weaker than Assumption  \ref{aspt:sparse} condition 2).  In particular, it can include important cases where we only have upper and lower bounds on each of $w^*$'s components, and $A$ is known to be a diagonal matrix. These cases are not covered by Assumption \ref{aspt:sparse}. On the other hand, the spectrum profile of $A$ will be required to choose the regularization parameter as shown in the following proposition.  
\blue{
\begin{prop}
	\label{prop:spetrum}
	By the following inequalities,
	\[
	\|w^*\|^2_A\leq \text{tr}(A)\|w^*\|^2_{A,S},\quad \|w^*\|^2_{A,\lambda}\leq \|w^*\|^2_{A,S}\sum_{i=1}^p \lambda\wedge \lambda_i.
	\]
	Assumption \ref{aspt:trace} implies that  $\|w^*\|_A =\tildeO(1)$ and $\|w^*\|^2_{A,\lambda} =\tildeO \left(\sum_{i=1}^p \lambda\wedge \lambda_i\right).$
	Moreover, under the conditions in Corollary \ref{cor:selectpara_online} and Assumption \ref{aspt:trace},  given any $\epsilon>0$, if the eigenvalues of $A$ follows,
	\begin{enumerate}[1)]
		\item Exponential decay: $\lambda_i=e^{-ci}$ for some constant $c>0$, and setting
	\[
		\lambda=\tildeO(\frac{\epsilon}{|\log \epsilon|}), \
		\delta(\epsilon)=\tildeO\left(\frac{\epsilon^3}{|\log \epsilon|^2}\right),\ \eta(\epsilon)=\tildeO\left(\frac{\epsilon^2}{|\log \epsilon|}\right),\ N(\epsilon)=\widetilde{\Omega}\left(\frac{|\log \epsilon|^3}{\epsilon^3}\right), 
		\] 
		we have  $\E [ G(w_N)1_{\tau\geq N} ] \leq 4\epsilon$.		
		\item Polynomial decay: $\lambda_i=i^{-c}$ for some constant $c>0$,  and setting
		\[
		\lambda(\epsilon) = \tildeO\Big( \epsilon^{
			\frac{c+1}{c}} \Big), \ \delta(\epsilon) = \tildeO\Big( \epsilon^{\frac{3c+2}{c}} \Big), \ \eta(\epsilon)  = \tildeO\Big( \epsilon^{\frac{2c+1}{c}} \Big), \ N(\epsilon) = \widetilde{\Omega}\Big(\frac{|\log(\epsilon)|}{\epsilon^{\frac{3c+2}{c}}}\Big),
		\]
		we have  $\E [ G(w_N)1_{\tau\geq N} ] \leq 4\epsilon$.  		
	\end{enumerate}
     In both cases,  we have
	\begin{align*}
	\E [ G(w_{N\wedge \tau}) ] \leq &\E [ G(w_0) ]+ \tildeO(\epsilon |\log \epsilon|). 
	\end{align*}
	So if $\calD=\{w: G(w)\leq (1+a)\E [G(w_0)] \}$, then 
	\[
	\Prob(\tau<N)\leq  \frac{1}{1+a} + \tildeO(\epsilon |\log \epsilon| ). 
	\] 
\end{prop}
}

We remark that the parameter of the spectrum decay (e.g., the constant $c$ in polynomial decay spectrum) is often assumed to be known for many functional data analysis problems \citep{HH07,CH08}. From Proposition \ref{prop:spetrum}, for both exponential decay and polynomial decay of the Hessian spectrum,  the sample size $N$ only depends on $p$ in a poly-logarithmic factor. 

\blue{ Similar to Proposition \ref{prop:sparse}, 
the second part of this result demonstrates how to bound the probability of SGD escaping the convexity region $\calD$ if it is the sub-level set with $w^*$ inside. 
The parameter $a$ controls the size of $\calD$. A larger $a$ produces a larger $\calD$ and hence a smaller escape probability.  }

\begin{rem}
	\label{rem:comparewithBLLT}
	It is interesting to compare our low effective dimension settings with the conditions used in \cite{BLLT:19}. For their main result, Theorem 4 in \cite{BLLT:19}, to yield dimension-independent generalization error bound, three conditions (formulated in our notation) need to hold: 1) $\|w^*\|^2$  is bounded by a constant; 2) $\tr(A)$ is bounded by a constant; 3) the spectrum of $A$ decays not so fast so that, for some $k$, $\sum_{i\geq k} \lambda_i\geq bN \lambda_k$ with some constant $b$. In comparison, our Assumption \ref{aspt:sparse} only requires conditions 1) and 2), but not the technical condition 3). Moreover, our result can also work under Assumption \ref{aspt:trace} where only $\|w^*\|^2_{A,S}$, instead of $\|w^*\|^2$, needs to be bounded. 
\end{rem}

\section{Overparameterization in Linear Regression}
\label{sec:linear}
In general, overparameterization may lead to overfitting, but this sometimes can be avoided. Our main result, Theorem \ref{thm:sgd}, provides a general tool to understand why overfitting sometimes happens and sometimes  does not. In this section, we will demonstrate how to apply our results on linear regression models in various high dimensional settings. This section is technically straightforward and is mainly used for pedagogical purpose.  The discussions on more technically challenging  cases for nonlinear and  non-convex models are provided in the next section.

\subsection{Linear regression}
First of all, we will find out the problem related parameters in Theorem \ref{thm:sgd} when applying to linear regression models. As in Section \ref{sec:SA}, we consider \emph{i.i.d.} data points form $\zeta_i = (x_i, y_i)\in \reals^p\times \reals$, where the response is generated by 
\begin{equation}\label{eq:linear_model}
y_i = x_i^{T} w^* + \xi_i.
\end{equation}
In \eqref{eq:linear_model}, $w^* \in \mathbb{R}^p$ is the true model-parameter to be estimated. $\xi_i\in \reals$ are observation noise terms in the observation process, 
and we assume they are \emph{i.i.d.} with zero mean and variance $\sigma^2$.  For simplicity, we assume that the data $x_i$ are \emph{i.i.d.}  Gaussian distributed, i.e., $x_i \sim \mathcal{N}(0,\Sigma)$. As a remark, our proof also  allows the non-Gaussian distribution with finite fourth moments. 

The regression loss of parameter $w$ on data $\zeta_i$ is 
\begin{equation}\label{eq:lossfun}
f(w, \zeta_i) = \frac{1}{2} (x_i^T w -y_i)^2.
\end{equation}
Plugging (\ref{eq:linear_model}) into (\ref{eq:lossfun}) and taking expectation, and we  find the population loss function
\begin{align}\label{eq:polossfun}
F(w) = \frac{1}{2} (w-w^*)^T\Sigma(w-w^*)+\dfrac{1}{2}\sigma^2.
\end{align}
Now we show the problem related parameters in Theorem \ref{thm:sgd} can be set as below:
\begin{prop}
	\label{prop:linregression}
	For linear regression, Assumption \ref{aspt:convex} holds with $A=\Sigma, \delta=0, \calD=\reals^p$. 
	When $w_0=0, \E G(w_0)=\frac12 \|w^*\|^2_\Sigma$, the stochastic gradient variance bounds in \eqref{eqn:sgvarianceless} and \eqref{eqn:sgvariance} hold with 
	\[
	r^2=2\sigma^2 \tr(\Sigma)+12\tr(\Sigma)\|w^*\|^2_\Sigma,\quad c_r=\frac{6}{\sigma^2}\max\{\|\Sigma\|,1\}. 
	\]
\end{prop}
As a consequence, by Proposition \ref{prop:sparse}, Assumption \ref{aspt:sparse} holds if $\|w^*\|^2$ and $\tr(\Sigma)$ are $O(1)$ constants, and the necessary sample size $N(\epsilon)$ in Corollary \ref{cor:selectpara_online} is independent of $p$.
Similarly, by Proposition \ref{prop:spetrum}, Assumption \ref{aspt:trace} holds if $\|w^*\|_{A,S}^2$ and $\tr(\Sigma)$ are $\tildeO(1)$ constants, and the sample size $N(\epsilon)$ depends on $p$ only via a polynomial logarithmic factor.

%

\subsection{High dimensional data with principle components}
\label{sec:motivate}

The low effective dimension settings in Section \ref{sec:nooverfit} naturally rise in many high dimensional problems.
For example, in image processing or functional data analysis (see e.g. \citealp{RS05, HH07, CH08, FJR15, XY18}), the data are in general assumed to take place in a  Hilbert space $(H, \langle \,\cdot\,,\,\cdot\,\rangle)$ with potentially infinitely many orthonormal basis functions $\{e^j,j=1,2,\ldots\}$. Each data can be written as 
\begin{equation}\label{eq:data}
x_i=\sum_{j=1}^\infty a^j_i e^j. 
\end{equation}
Suppose $a^j_i$ are independent Gaussian random variables with mean zero and variance $\sigma_j^2$. Note that
$
\E \langle x_i, x_i\rangle=\sum_{j=1}^\infty \sigma_j^2.
$ 
Therefore, for each data $x_i\in H$, we assume that $\sum_{j=1}^\infty \sigma_j^2<\infty$ so that the norm of the data is bounded, which implicitly requires $\sigma_j$ decaying to zero \citep{HH07}. Given the form of $x_i$ in \eqref{eq:data}, the linear regression model takes the following form,
\begin{equation}
\label{eqn:inf}
y_i=\langle x_i,w^*\rangle+\xi_i,
\end{equation}
where $w^*=\sum_{j=1}^\infty w^{*,j} e^j$. If we assume $w^*\in H$, then $\langle w^*, w^*\rangle=\sum_{j=1}^\infty (w^{*,j})^2<\infty$. 

When training this ``infinite dimensional'' linear regression model in \eqref{eqn:inf}, we would need  a finite projection $\calP_p: H\mapsto \reals^p$.   When the basis functions are available, one natural choice of the projection is
\[
\calP_p x_i=\calP_p\left(\sum_{j=1}^\infty a^j_i e^j\right):=[a^1_i, \ldots a^p_i]^T.
\]
Then the $p$-dimensional linear regression model is formulated as
\begin{equation}
\label{eqn:projp}
y_i=(\calP_p x_i)^T w^*_p+\xi_i^p.
\end{equation}
It is worthwhile noticing that the true infinite dimensional model in \eqref{eqn:inf} is compatible with the finite dimensional model in \eqref{eqn:projp}, in the sense that 
\[
w^*_p=\calP_p w^*=[w^{*,1},\ldots, w^{*,p}]^T,\quad \xi_i^p=\xi_i+\sum_{j=p+1}^\infty w^{*,j}a^j_i.
\]
Since $a^j_i$ are independent Gaussian random variables, we have $\xi_i^p\sim \mathcal{N}(0,\sigma_{\xi,p}^2)$ with 
\[
\sigma^2_{\xi,p}:=\sigma^2+\sum_{j=p+1}^\infty \sigma_j^2(w^{*,j})^2\leq \sigma^2+\|w^*\|_{\Sigma}^2, \quad \|w^*\|_{\Sigma}^2:=\sum_{j=1}^\infty \sigma_j^2(w^{*,j})^2.
\]
In the finite dimensional model \eqref{eqn:projp}, the data  $\calP_p x_i$ has the population covariance matrix $\Sigma_p=\text{diag}(\sigma_1^2,\ldots, \sigma_p^2),$ whose trace is bounded by 
\[
\tr(\Sigma_p)=\sum_{j=1}^p \sigma_j^2 \leq \sum_{j=1}^\infty \sigma_j^2.
\] 
Therefore, by Proposition \ref{prop:linregression}, the problem related parameters in Theorem \ref{thm:sgd} are
\begin{align}
\notag
&A_p=\Sigma_p,\quad\text{with}\quad \tr(A_p)\leq \sum_{j=1}^\infty \sigma_j^2,\quad \|A_p\|=\sigma_1^2,\\
\label{eq:trace_A_p}
&r_p^2=2\sigma^2_{\xi,p}\tr(\Sigma_p)+12\tr(\Sigma_p)\|w_p^*\|^2_{\Sigma_p}\leq 2(\sigma^2+7\|w^*\|^2_\Sigma)\sum_{j=1}^\infty \sigma_j^2,\\
\notag
&  c_{r,p}^2=\frac{6}{\sigma^2_{\xi,p}}\max\{1,\sigma_1^2\}\leq \frac{6}{\sigma^2}\max\{1,\sigma_1^2\}. 
\end{align}
Moreover, if we use $w_0=0$, then $\E \|w_0\|^2=0$ and 
\[
\E G(w_0)=\frac12\|w^*_p\|^2_{\Sigma_p}\leq \|w^*\|^2_\Sigma.
\]
Note that 
\[
\|w^*\|^2_\Sigma=\sum_{j=1}^\infty \sigma_j^2(w^{*,j})^2\leq \|w^*\|_\infty^2\sum_{j=1}^\infty \sigma_j^2,\quad \|w^*\|_\infty:=\max_{1\leq j}|w^{*,j}|.
\] 
So as long as $\|w^*\|_\infty$ is finite, the upper bounds above are independent of dimension $p$.

When the true loading parameter $w^*$ is an element of $H$, $\|w^*_p\|^2\leq \langle w^*, w^*\rangle<\infty$. Then we can check that all items of Assumption \ref{aspt:sparse} hold. So by Proposition \ref{prop:sparse}, we know that the generalization error is dimension-independent. Moreover, this does not require any information of the spectrum decay profile.  

More generally, we only need that each component of the true loading parameter $w^*$  is bounded, and $w^*$ does not need to be an element of $H$ itself. In particular, we note that
\[
\|w^*_p\|_{\Sigma_p,S}=\max_{1\leq j\leq p} |w^{*,j}| \leq \|w^*\|_\infty.
\]
Therefore, if $\|w^*\|_\infty$ is finite, Assumption \ref{aspt:trace} holds  (but in general Assumption \ref{aspt:sparse} does not). Then by Proposition \ref{prop:spetrum},  the generalization error can be dimension independent when we know the spectrum decay profile. 

%
%
\begin{figure}[!t]
	\centering
	\begin{subfigure}{.45\textwidth}
		\includegraphics[width=1.1\textwidth]{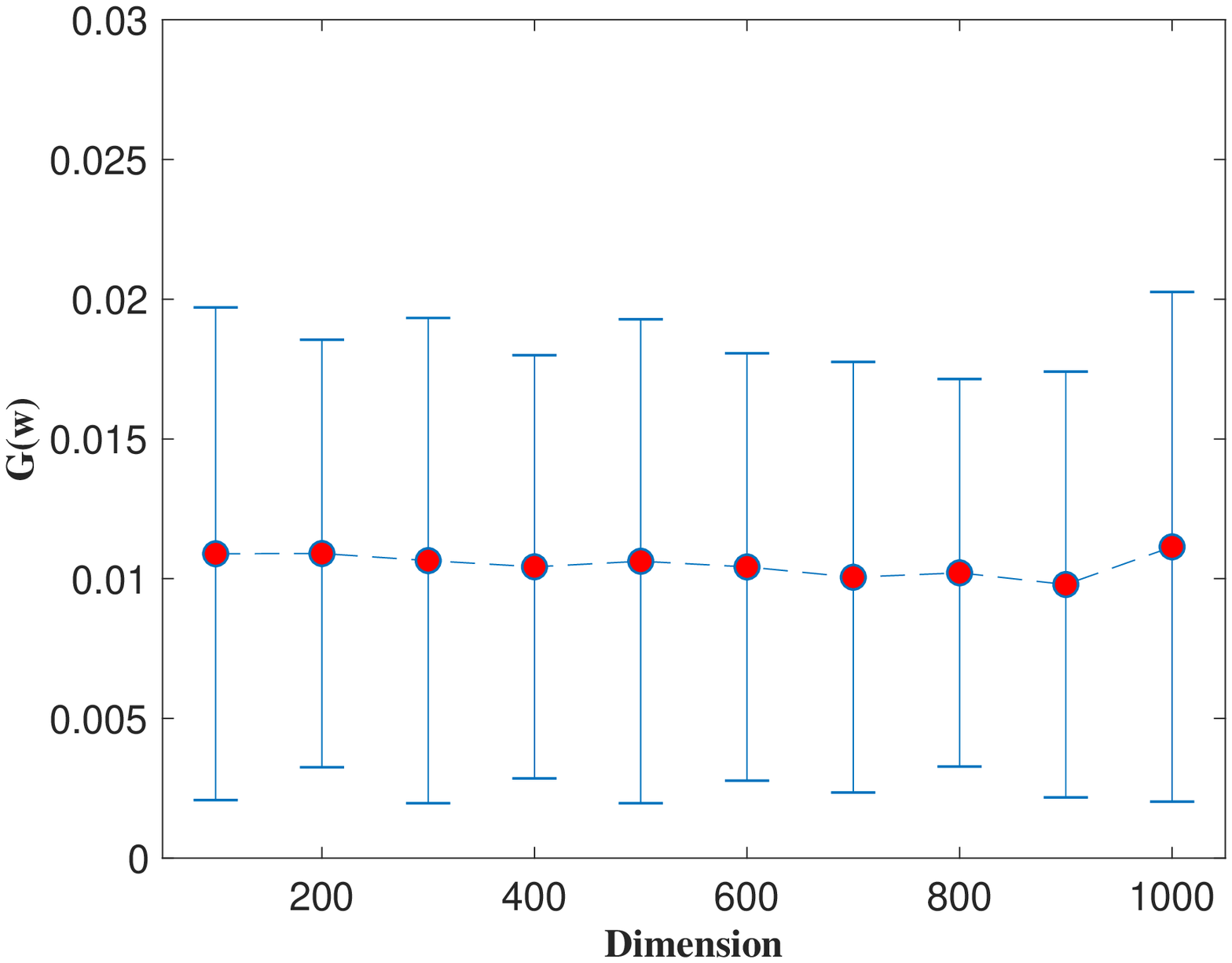}
		\caption{$w^{*,j}=j^{-1}$, $\sigma^2_j=j^{-2}$}
	\end{subfigure}
	\begin{subfigure}{.45\textwidth}
		\includegraphics[width=1.1\textwidth]{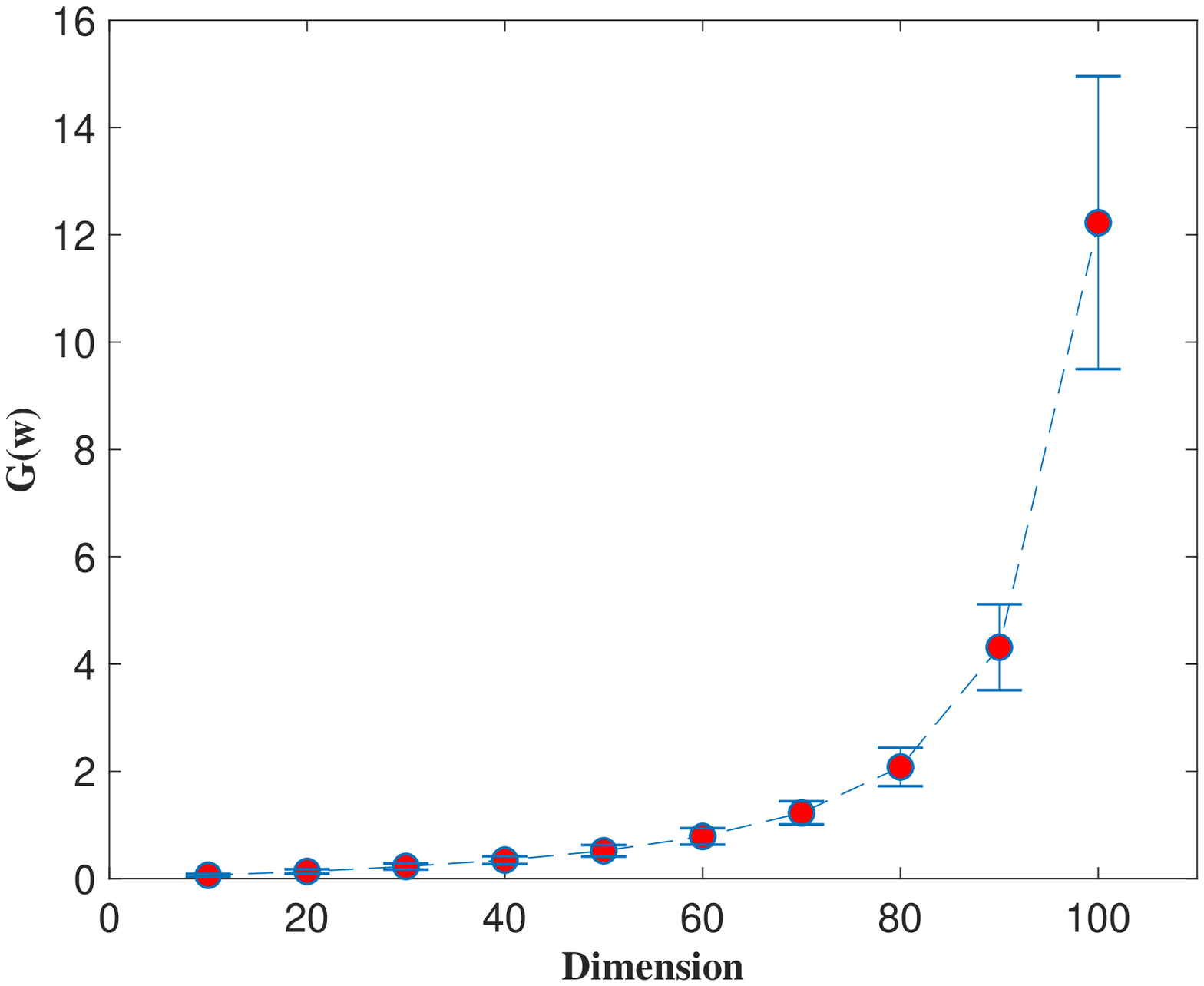}
		\caption{$w^{*,j}=j^{-1}$, $\sigma^2_j\equiv 1$}
	\end{subfigure}
	\begin{subfigure}{.45\textwidth}
		\includegraphics[width=1.1\textwidth]{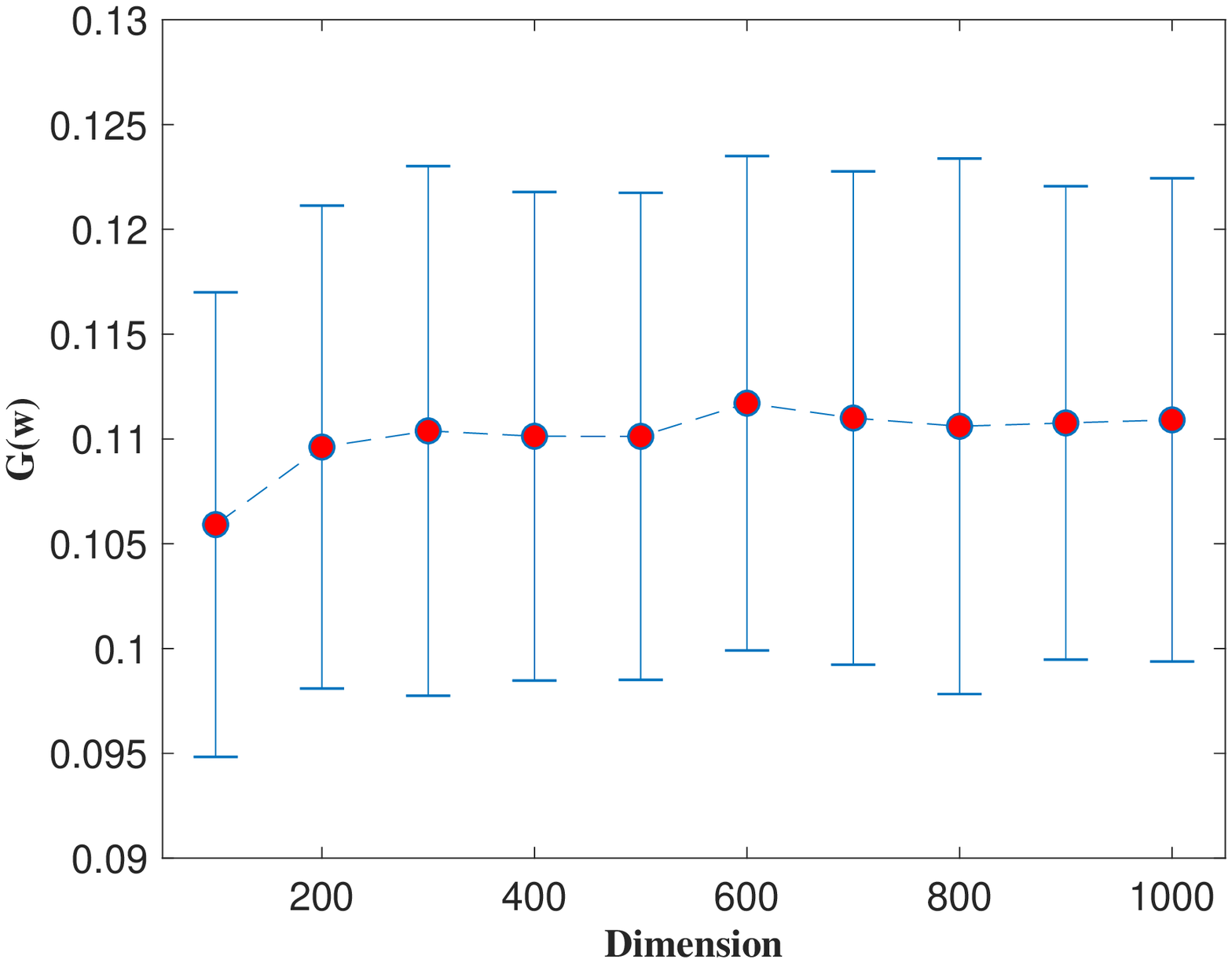}
		\caption{$w^{*,j}\equiv 1$, $\sigma^2_j=j^{-2}$}
	\end{subfigure}
	\begin{subfigure}{.45\textwidth}
		\includegraphics[width=1.1\textwidth]{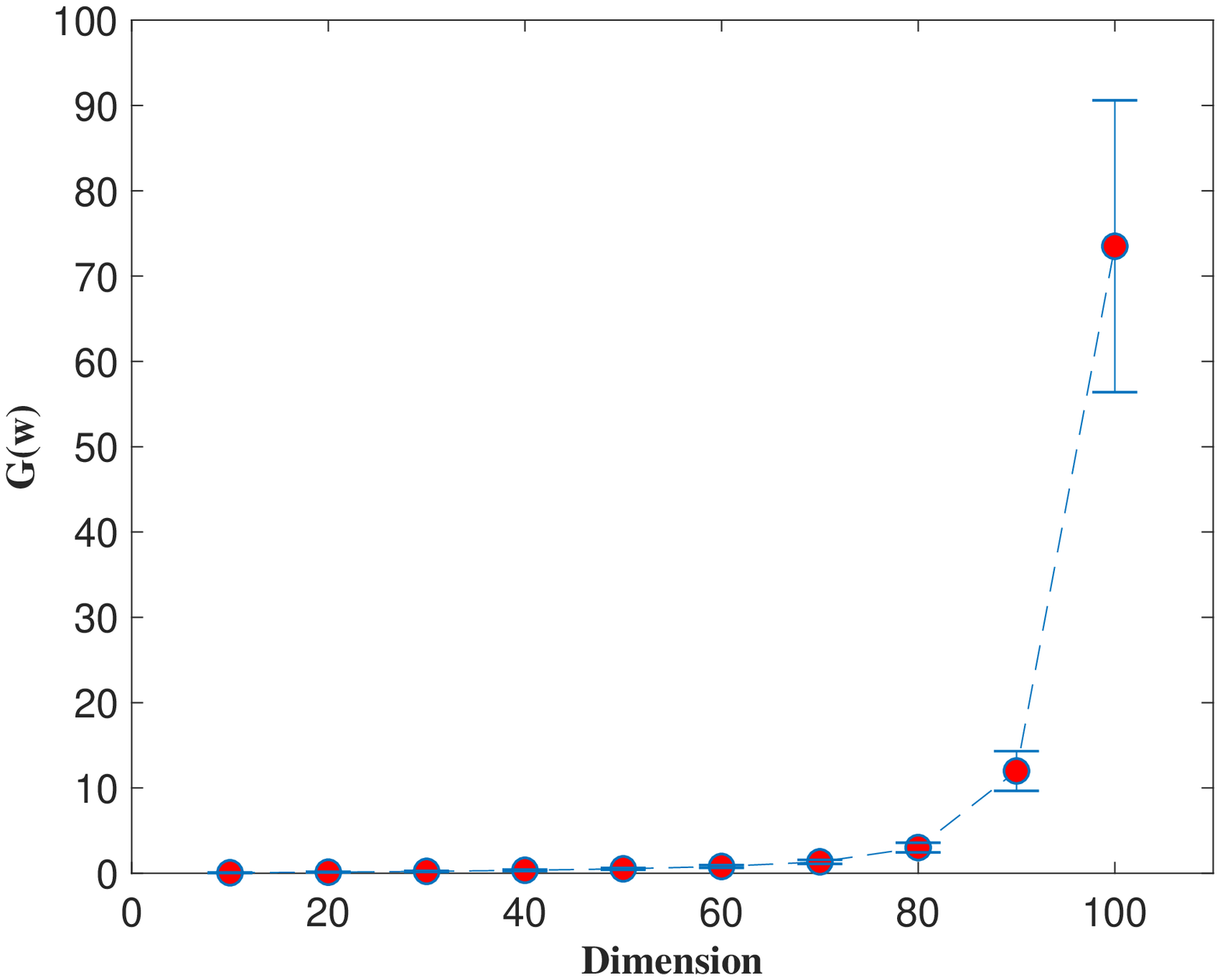}
		\caption{$w^{*,j}\equiv 1$, $\sigma^2_j \equiv 1$}
	\end{subfigure}
	\caption{Generalization error bar plot with high dimensional linear regression for different settings of  $w^*_p$ and $\Sigma_p$.   The $x$-axis is the dimension and $y$-axis is the generalization error. } 
	\label{fig:intro}
\end{figure}

As a simple demonstration, we run some simulations of SGD on linear regression model \eqref{eqn:projp} and present them in Figure \ref{fig:intro}. We run SGD on \eqref{eqn:projp} with the sample size $N=500$ and the regularization parameter $\lambda=0.01$. The covariance spectrum of predictors is set to be $\sigma^2_j=j^{-2}$ so that $\tr(A_p)$ in \eqref{eq:trace_A_p} is a constant, and the true parameter is set to be $w^{*,j}=j^{-1}$ for $1 \leq j \leq p$ so that $\|w^*\|_{\Sigma}$ is bounded. The problem dimension ranges from $p=100$ to $p=1000$, which can be larger than the sample size.  We use the final SGD output $w_{500}$ as the estimator and compute the generalization error as in \eqref{eqn:linloss1}. We repeat this experiment $1000$ times and compute the mean and standard deviation. We plot the error bar plot in  the upper left panel of Figure \ref{fig:intro}. As one can see, the generalization error does not increase as the dimension increases, even when $p \gg N$. As a comparison experiment, we run simulations with the same settings except for $\sigma^2_j\equiv 1$. We plot the generalization error in the upper right panel of Figure \ref{fig:intro}, which clearly shows the overfitting phenomenon, even when the dimension is in a lower range. 

The similar story repeats when the true parameter $w^{*,j}\equiv 1$ (i.e., the case when $\|w^*\|_{\infty}$ is bounded), where the plots are given by the lower panels in Figure \ref{fig:intro}. As one can see, the generalization error with the decaying spectrum still remains stable against the increase of dimension, and it does not change much from the previous setting where the components of $w^*$ are decaying. Meanwhile, overfitting with constant spectrum (i.e., $\sigma^2_j\equiv 1$) becomes stronger. This simple illustrative example justifies that the low effective dimension helps to address the overfitting issue, even when the dimension $p$ is much larger than the sample size $N$.

\subsection{Overfitting with redundant features}
Another interesting setting of overparameterization is to consider adding redundant predictors to an existing model. In this scenario, the  true model is low dimensional with the true parameter $w^*\in \reals^{d}$. Suppose  we do not know the true model and collect additional features $z \in \reals^{p-d}$, so that the overparameterized linear model is written as
\begin{equation}
\label{eqn:redund}
y_i=x_i^Tw^*+z_i^T u^*+\xi_i,
\end{equation}
where $w^* \in \reals^d$, $u^*=0$, and  $[x_i;z_i]$ is jointly Gaussian with mean zero and covariance 
\[
\Sigma_p=\begin{bmatrix} \Sigma_x & B\\ B^T  &\Sigma_z \end{bmatrix}.
\]
We assume that $\|\Sigma_z\|\leq \|\Sigma_x\|$ for the ease of discussion. 
Then, $\Sigma_p$ being PSD implies that $\|B\|\leq \|\Sigma_x\|$. Since we do not impose any restriction on $B$ other than $\Sigma_p$ being PSD,   our setting allows the possibility that some components of $z_i$ to be highly correlated or even identical with the ones of $x_i$. This, in general, leads to highly singular design matrices and unstable offline learning results. 


We apply Proposition \ref{prop:linregression} and find $A_p=\Sigma_p$. By triangular inequality, for any vectors $x$ and $z$,
\[
\|\Sigma_p[x;z]\|=\|[\Sigma_x x+Bz; B^Tx+\Sigma_z z]\|\leq 2\|\Sigma_x\|\|[x;z]\|\Rightarrow \|\Sigma_p\|\leq 2\|\Sigma_x\|. 
\]
For simplicity, we initialize with $[w_0;u_0]=0$, so 
\[
\E G(w_0,u_0)=\frac12\|w_0-w^*\|^2_{\Sigma_x}+\frac12\|u_0\|_{\Sigma_z}^2=\frac12\|w^*\|^2_{\Sigma_x}.
\]
Moreover, we have
\[
c_{r,p}=\frac{6}{\sigma^2}\max\{\|\Sigma_p\|,1\}\leq \frac{6}{\sigma^2}\max\{2\|\Sigma_x\|,1\},
\]
and
\[
\|[w^*,u^*]\|_{\Sigma_p,S}=\|w^*\|_\infty,\quad \|[w^*,u^*]\|^2=\|w^*\|^2.
\]
These upper bounds are all independent of $p$, or the choice of $\Sigma_z$ and $B$.

Meanwhile,
\begin{equation}\label{eq:r_p}
r^2_p=2(\sigma^2+6\|w^*\|^2_{\Sigma_x}) (\tr(\Sigma_x)+ \tr(\Sigma_z)),\quad \tr(A_p)=\tr(\Sigma_x)+\tr(\Sigma_z).
\end{equation}
Given these simple calculation, we find that the only problem related parameters that depend on $z$ are $r^2_p$ and $\tr(A_p)$ in \eqref{eq:r_p} through $\tr(\Sigma_z)$. 
Therefore, our theory indicates that there is a simple dichotomy on whether  the model \eqref{eqn:redund} will overfit. 

If $\tr(\Sigma_z)$ is bounded by a constant independent of $p$,  Proposition \ref{prop:sparse} applies, which indicates that the generalization error is also independent of the ambient dimension $p$. This can happen if we select data features in $z$ as PCA components. For example, suppose that the redundant data is in the form of $\sum_{j=1}^\infty a_{i}^j e^j$ as in the setting of  \eqref{eq:data}, and we collect the $p-d$ dimensional principle components as  $z_i=[a_i^1,\ldots, a_i^{p-d}]^T$. Then $\tr(\Sigma_z)=\sum_{j=1}^{p-d}\sigma_j^2<\sum_{j=1}^\infty \sigma_j^2$, which is independent of $p$.


If $\tr(\Sigma_z)$ grows with $p$,  model \eqref{eqn:redund} may overfit. For simplicity, we consider a special case where $\Sigma_z=I_{p-d}, B=0$. In other words, the redundant features are independent with each other and the features of $x$. Then our derivation shows that $r^2_p=O(p)$. This indicates that the learning results may overfit.

To demonstrate this dichotomy, we simulate the SGD learning results and present their generalization error in Figure \ref{fig:redun}. In particular, we set $d=5$ with true parameter $w^*=[1,1,1,1,1], \Sigma_x=I_5, \sigma^2=1$. We let $B=0$ and choose first that $\Sigma_z$ to be diagonal with decaying entries $\frac{1}{j^2}$. We run SGD with 500 iterations and compute the generalization error of the final iterate. We repeat this 1000 times and plot the  error bar plot in the left panel of Figure \ref{fig:redun}. As we can see, the generalization error is stable against the increase of the dimension $p$. In comparison, if we use $\Sigma_z=I_{p-d}$, the learning results overfit, as we can see from the right panel of Figure \ref{fig:redun}. 

\begin{figure}[!t]
	\centering
	\begin{subfigure}[t]{.46\textwidth}
		\includegraphics[width=1.1\textwidth]{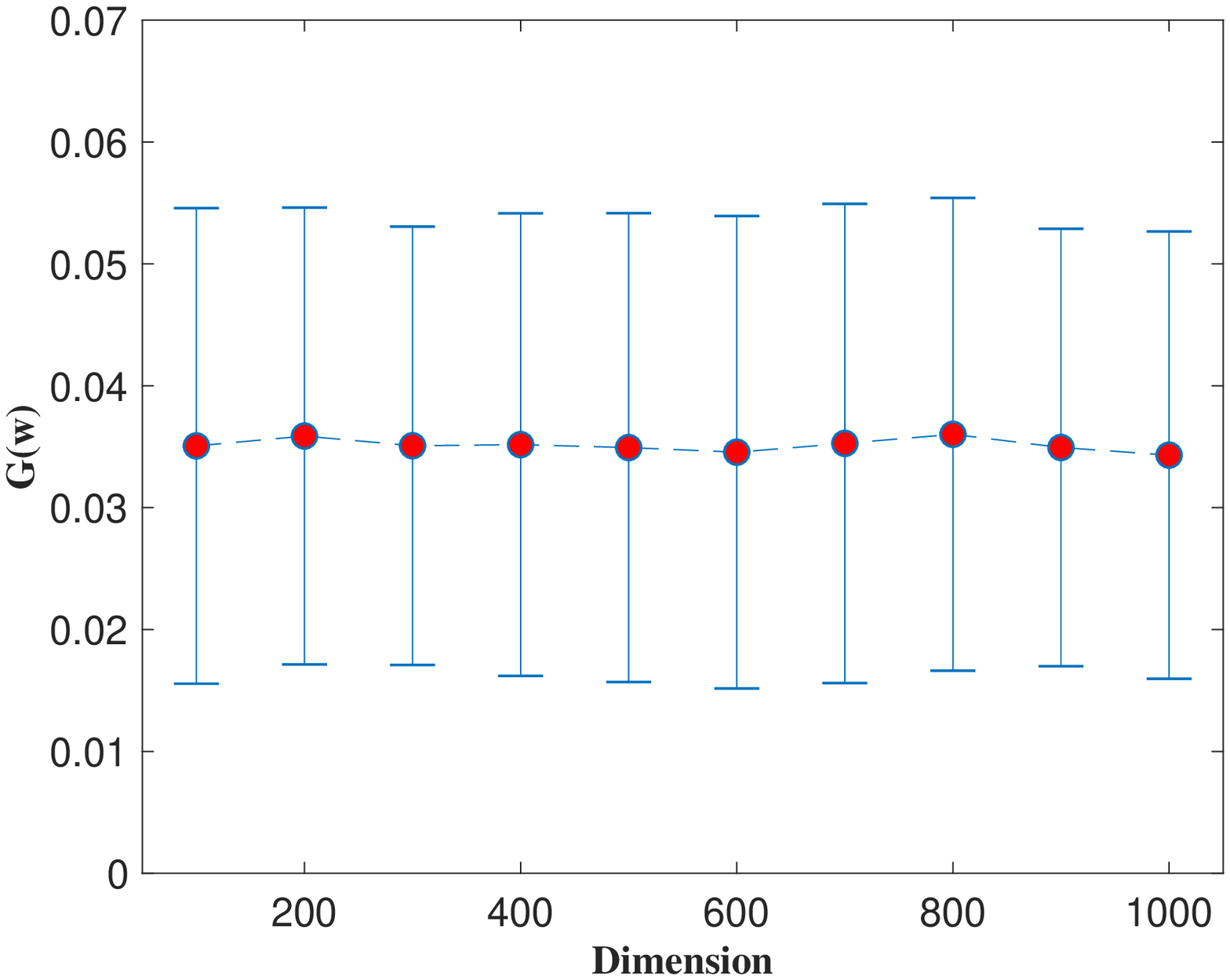}
		\caption{$\Sigma_z$ is a diagonal matrix with decaying entry $(\Sigma_{z})_{jj}=\frac{1}{j^2}$.}
	\end{subfigure}
	\begin{subfigure}[t]{.46\textwidth}
		\includegraphics[width=1.1\textwidth]{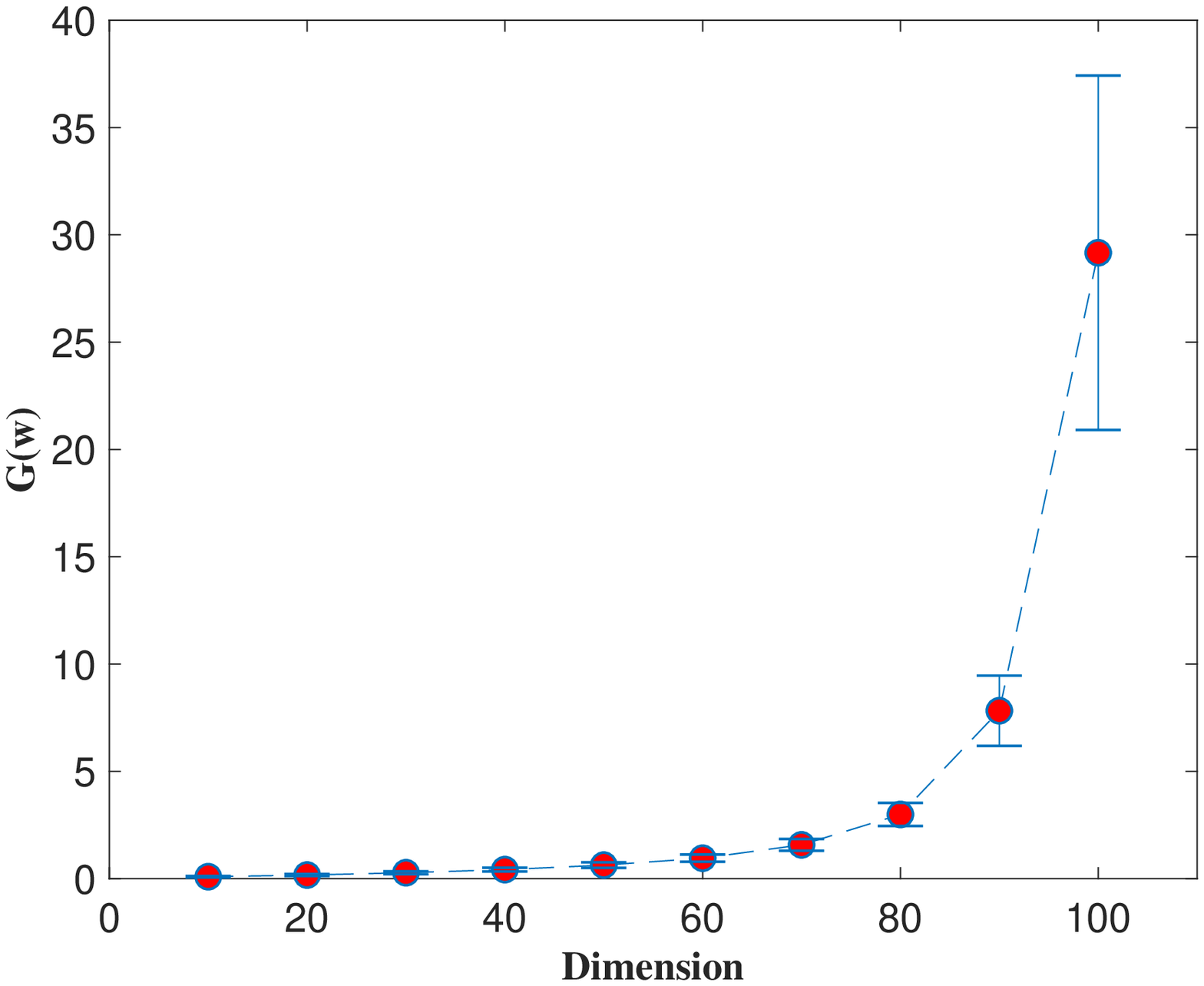}
		\caption{$\Sigma_z=I_{p-d}$.}
	\end{subfigure}
	\caption{Generalization error bar plot with high dimensional redundant features for two different cases of $\Sigma_z$. The $x$-axis is the dimension $p$ and $y$-axis is the generalization error. }
	\label{fig:redun}
\end{figure}

\section{Overparameterization for Nonlinear and Non-convex Models}
\label{sec:non_linear}

In this section, we apply our main theorem and corollaries in Section \ref{sec:nooverfit} to several important nonlinear and non-convex statistical problems, such as logistic regression, M-estimator with Tukey's biweight loss function, and two-layer neural networks.  
%
%
%

\subsection{Logistic regression}
We consider the logistic regression for binary classification with $N$ \emph{i.i.d.} data $\zeta_i=(x_i,y_i)$. The binary response $y_i$ takes values within $\{ -1, 1\}$ with probability 
\begin{align*}
\mathbf{P}(y_i=y|x_i) = \dfrac{1}{1+\exp(-yx_i^{T}w^*)},\quad y=\pm 1, 
\end{align*}
where  $w^{*} \in \reals^p$ is the true parameter to be estimated. We assume the predictors $x_i$ are \emph{i.i.d.} with  $\mathbb{E}x_ix_i^T = \Sigma$. 
For each data, we adopt the negative log-likelihood as the loss function
\begin{align*}
f(w, \zeta_i):=\log(1+\exp(-y_ix_i^{T}w)),
\end{align*}
and the corresponding population loss is given by 
\begin{align*}
F(w) = \mathbb{E}f(w, \zeta) &= \mathbb{E}\log(1+\exp(-yx^{T}w)).
\end{align*}

The problem related parameters in Theorem \ref{thm:sgd} can be set by the following proposition. 
\begin{prop}
	\label{prop:logit}
	For logistic regression, Assumption \ref{aspt:convex} holds with $A=\Sigma, \delta=0, \calD=\reals^p$. 
	When $w_0=0, \E G(w_0)=\log 2=O(1)$, the stochastic gradient variance bound in \eqref{eqn:sgvariance} holds with 
	\[
	r^2= \tr(\Sigma),\quad c_r=0.  
	\]
\end{prop}

As a consequence, by Proposition \ref{prop:sparse}, Assumption \ref{aspt:sparse} holds if $\|w^*\|^2$ and $\tr(\Sigma)$ are $O(1)$ constants, and the sample size $N(\epsilon)$ in Corollary \ref{cor:selectpara_online} is independent of $p$.
Similarly, by Proposition \ref{prop:spetrum}, Assumption \ref{aspt:trace} holds if $\|w^*\|_{\Sigma,S}^2$ and $\tr(\Sigma)$ are $\tildeO(1)$ constants, and the sample size $N(\epsilon)$ depends on $p$ only via a poly-logarithmic factor.

\subsection{$M$-estimator with Tukey's biweight loss function}

In this non-convex example, we assume that the data $\zeta_i=(x_i, y_i)$ are generated from a linear model
\begin{equation}\label{eq:tukey_model}
y_i=x_i^T w^*+\xi_i.
\end{equation}
We assume $x_i\sim \mathcal{N}(0,\Sigma)$, and $\xi_i$ are i.i.d. mean-zero noises with finite fourth moment.  We adopt the \emph{non-convex} Tukey's biweight loss function as follows for the purpose of robust estimation
\[
\rho(u)=\begin{cases}
\frac{c^2}6[1-(1-(u/c)^2)^3] \quad &\text{if } \ |u|\leq c;\\
\frac{c^2}6\quad \quad&\text{if }\ |u|>c. \\
\end{cases}
\]
Then the individual data loss function  and the population loss are given by,
\[
f(w,\zeta)=\rho(x^Tw-y )=\rho(x^T (w-w^*)-\xi),\quad  F(w)=\E\rho(x^T (w-w^*)-\xi).
\]
\begin{prop}
	\label{prop:Mest}
	For the $M$-estimator with Tukey's biweight loss in \eqref{eq:tukey_model},  the model true parameter $w^*$ is a local minimum if and only if
	\[
	c_0=\E [(1-(\xi/c)^2)(1-5(\xi/c)^2)1_{|\xi|\leq c}] >0.
	\] 
	In that case,  Assumption \ref{aspt:convex} holds with any $\delta\geq 0$, $A=\Sigma$ and
	\[
	\calD=\{w: \|w-w^*\|_\Sigma \leq \tfrac{c_0+\delta}{16}\}.
	\]
	Moreover, the stochastic gradient variance bound in \eqref{eqn:sgvariance} holds with 
	\[
	r^2= \tr(\Sigma),\quad c_r=0.  
	\]
\end{prop}

Since $G(w_0)\leq \max_u \rho(u)=\frac{c^2}{6}$, by Proposition \ref{prop:sparse}, Assumption \ref{aspt:sparse} holds if $\|w^*\|^2, \|w_0\|^2$ and $\tr(\Sigma)$ are $O(1)$ constants, and the sample size $N(\epsilon)$ in Corollary \ref{cor:selectpara_online} is independent of $p$.
Similarly, by Proposition \ref{prop:spetrum}, Assumption \ref{aspt:trace} holds if $\|w_0\|_\infty,\|w^*\|_{\Sigma,S}^2$ and $\tr(\Sigma)$ are $\tildeO(1)$ constants, and the sample size $N(\epsilon)$ depends on $p$ only via a polynomial logarithmic factor. 

\subsection{Two-layer neural network}
\label{sec:NN}

In this example, we consider applying our result to two-layer neural networks (NN).  
We assume that every data point $\zeta = (x, y)$ consists of a  $p$-dimensional predictor $x \sim \mathcal{N}(0,\Sigma)$ and  a univariate response $y \in \reals$. We assume that the response is generated by
\[
y=g(w,x)+\xi, \quad \E\xi=0, \quad \E\xi^2=\sigma^2_0.
\]
The function $g$ takes the form of a two-layer NN:
\begin{equation}\label{eqn:g2NN}
g(w,x)=c^T\psi(b x+a)=\sum_{i=1}^k c_i \psi(b_i^Tx+a_i).
\end{equation}
In \eqref{eqn:g2NN}, $a$ and $c$ are $k$-dimensional vectors with $a_i$ and $c_i$ being their components. The notation $b$ is a $p$ by $k$ matrix, and $b_i$ denotes the $i$-th column of $b$ with  $i=1,\ldots, k$. We impose no restriction on $k$ and it can depend on $p$ in general.  We denote all the parameters by $w=[a;b_1,\ldots, b_k;c]\in \reals^{(p+2)k}$.
In \eqref{eqn:g2NN}, $\psi$ denotes the activation function. Popular choices of $\psi$ include the rectified linear unit (ReLu), sigmoid function, and the hyperbolic tangent. Here, we do not require $\psi$ to take a specific form but only satisfy certain regularity assumptions for some constant $C>0$,
\begin{equation}\label{eqn:psi}
\psi(0)=0,\quad |\dot{\psi}(x)|\leq C,\quad |\ddot{\psi}(x)|\leq C.
\end{equation}
It is easy to verify that hyperbolic tangent satisfies these requirements, and the sigmoid also satisfies these if we shift its center to zero. The condition $\psi(0)=0$ is mainly for the ease of technical derivations. Although ReLu does not have continuous derivatives, one can find a smooth approximation to meet these requirements.

Since we consider the regression problem, the squared loss function is given by
\[
f(w,\zeta)=(y-g(w,x))^2=(g(w^*,x)+\xi -g(w,x))^2. 
\]
We also introduce the following $(p+2)k$ by $(p+2)k$ block-diagonal matrix 
\[
\Sigma^\star=\text{diag}\{I_k, \Sigma,\Sigma\ldots, \Sigma, I_k\}.
\]
This matrix introduces a high dimensional norm 
\[
\|w\|_{\Sigma^\star}^2=w^T\Sigma^\star w=\|a\|^2+\sum_{i=1}^k \|b_i\|^2_{\Sigma}+\|c\|^2.
\]
Recall that $b_i$ is of dimension $p$, its contribution to $\|w\|_{\Sigma^\star}^2$ is  $\|b_i\|^2_\Sigma$. By Proposition \ref{prop:spetrum}, $\|b_i\|^2_\Sigma\leq \tr(\Sigma)\|b_i\|^2_{\Sigma,S}$, which can be independent  of $p$ under suitable conditions.

We are ready to show that  the two-layer NN will not overfit in some overparameterized settings. 

\begin{prop}
	\label{prop:2LNN}
	Assume the activation function satisfies the condition in \eqref{eqn:psi}. 
	With the two-layer NN defined in \eqref{eqn:g2NN},  Assumption \ref{aspt:convex} holds for any $\delta \in (0, 1/4]$ with
	\[
	A=C_0(w^*)\Sigma^\star,\quad \calD=\{w: \|w-w^*\|_{\Sigma^\star} \leq \delta C_1(w^*)\|w^*\|_{\Sigma^\star}\}.
	\]
	For any $w_0\in \calD, G(w_0)\leq C_2(w^*)\|w^*\|_{\Sigma^\star}^4$, the stochastic gradient variance bound in \eqref{eqn:sgvariance} holds with 
	\[
	r^2= C_3(w^*), c_r=0.
	\] 
	The exact values of the problem related parameters are given by 
	\begin{align*}
	&C_0(w^*)=7C^2\|w^*\|^2_{\Sigma^\star}, \quad
	C_1(w^*)=\frac{2}{9\sqrt{2}(2\|w^*\|_{\Sigma^\star}+1)}, \quad C_2(w^*)=C^2\|w^*\|^4_{\Sigma^\star},\\
	&C_3(w^*)=8\sqrt{3}(1+\tr(\Sigma))C^2 \|w^*\|^2_{\Sigma^\star}\big(C^2 \|w^*\|_{\Sigma^\star}^4+\sigma_0^2\big).
	\end{align*}
\end{prop}
As a consequence, by Proposition \ref{prop:sparse}, Assumption \ref{aspt:sparse} holds if 
\begin{equation}
\label{tmp:cond1}
\max\left\{\|a^*\|^2+\sum_{i=1}^k\|b^*_i\|^2+\|c^*\|^2,\tr(\Sigma),\|w_0\|_\infty\right\}=O(1),
\end{equation}
and the sample size $N(\epsilon)$ is independent of $p$. Similarly, by Proposition \ref{prop:spetrum}, Assumption \ref{aspt:trace} holds if 
\begin{equation}
\label{tmp:cond2}
\max\left\{k, \tr(\Sigma),\|w_0\|_\infty, |a_i|, |c_i|, |v_j^Tb_i|,  i=1,\ldots,k,j=1,\ldots,p\right\}=\tildeO(1),
\end{equation} 
where $v_j$ are the eigenvectors of $\Sigma$. Then the sample size $N(\epsilon)$ depends on $p$ only via a polynomial logarithmic factor. 

It is worthwhile mentioning that a similar version of  Condition \eqref{tmp:cond1} can also be found in \cite{NLBLS:19}. In particular, \cite{NLBLS:19} assumed $\|c^*\|^2,\sum_{i=1}^k \|b^*_i\|^2$ to be $O(1)$ while the parameter $a$ is set to be $0$. There is no variance assumption of $x_i$ in \cite{NLBLS:19}, but it is assumed that $\E\|x_i\|^2$ is $O(1)$, which is equivalent to requesting $\tr(\Sigma)=O(1)$ in our setting.

When the width of the hidden layer $k$ is a fixed constant, the second condition \eqref{tmp:cond2} is in general less restrictive than the first one \eqref{tmp:cond1}, since it allows $\|b^*_i\|$ to grow with $p$.  When $k$ grows with $p$, only the first condition is applicable, and it requires that $\|a^*\|^2+\sum_{i=1}^k\|b^*_i\|^2+\|c^*\|^2$ is bounded by $O(1)$. In other words, we need either $k=O(1)$ or the true parameters to be bounded by $O(1)$ to prevent overfitting. This can also be understood intuitively. Note that the  output  of the two-layer NN in  \eqref{eqn:g2NN} is a sum of $k$ objects. Therefore, if $k$ grows with $p$, the output of \eqref{eqn:g2NN}  will diverge, which contradicts the common assumption that $g$ is bounded (see, e.g., \citealp{LL2018, NLBLS:19, ALL:19}).  Our result is consistent with the results in  \cite{ALL:19}, in the sense that \cite{ALL:19} also showed that the sample size needs to grow with $k$. It is also possible to rescale $g$ by multiplying \eqref{eqn:g2NN} with a factor $\frac{1}{k}$ or $\frac{1}{\sqrt{k}}$, as done by \cite{ADHLW:19},  so that the generalization error is independent of the parameter $k$.  
\section{Conclusions and Future Works}
\label{sec:conclusion}
One classical canon of statistics is that high dimensional models are prone to overfitting when the data sample size is not sufficiently large. 
However, many existing models, such as neural networks (NN), exhibit stable generalization error despite being overparameterized.
This paper developed an analysis framework of the generalization error for high dimensional regularized online learning.
The error bound can be  interpreted as a bias-variance tradeoff through a simplified stochastic approximation.
This result indicates that overparameterization does not lead to overfitting if the model has a low effective dimension. 
We demonstrated how to apply this framework on various models such as linear regression, logistic regression, $M$-estimator with Tukey's biweight loss, and two-layer NN. 


There are a few future directions. First,  the generalization bound in Theorem \ref{thm:sgd} only applies when the SGD iterates stay in a local region $\calD$ near the true parameter $w^*$. Although it is a common assumption when analyzing non-convex models,  this assumption might  be difficult to check in practice. 
To address this challenge, a simple data splitting can serve as a remedy. In practice, we may split the data into two parts and run SGD through the first half.  If $w_{N/2}$ gives us some rough ideas on how to construct $\calD$, we can use $w_{N/2}$ as an  initialization and run SGD on the second half of the data.
In the process, we can check whether the iterates stay in $\calD$, and then Theorem \ref{thm:sgd} provides us an upper bound for the generalization error of $w_N$.  Of course, it would be an interesting future work on how to construct $\calD$ based on the SGD solution from the first half of the data.

Second, our framework indicates that the ambient model dimension itself may not be a good indicator of model complexity,  especially in  overparameterized settings. 
The quantity that characterizes the data variability  may lead to new information criterion for model selection in the overparameterized setting. Such results may extend the classical criteria such as the AIC and BIC.

\blue{ Last but not least, regularization and overparameterization are good tools to handle misspecified models. \citet{Hastie19}
has discussed this issue for linear regression problems. How to extend our results to nonlinear misspecified models will be very interesting.  }


\section*{Acknowledgement}
Xi Chen would like to thank the support from NSF via the grant IIS-1845444.
Qiang Liu would like to thank the support from Singapore MOE via the grant R-146-000-258-114.
Xin T. Tong would like to thank the support from Singapore MOE via the grant R-146-000-292-114. 
The authors also thank the suggestions made by the editor and anonymous reviewers.


\newpage

\appendix

\section{Proof of the Main Results in Section \ref{sec:main}}

\subsection{Preliminaries}
\begin{lem}
	\label{lem:norm}
	For any vector $v\in \reals^p$ and PSD matrix $A\in \reals^{p\times p}$, the following results hold
	\begin{enumerate}[1)]
		\item For any $-\delta A\preceq B\preceq A $, let $B=V \Lambda V^T$ be the eigenvalue decomposition of $B$, and denote $|\Lambda|$ as taking absolute value on each element of the diagonal matrix $\Lambda$.  Denote $ |B| = V |\Lambda| V^T$. Then for any vectors $v$ and $w$, $a>0$
		\[
		2\langle v, B w\rangle \leq a\langle v, |B| v\rangle+\frac{1}{a}\langle w, |B| w\rangle\leq a(1+\delta)\|v\|^2_A+\frac{1+\delta}{a}\|w\|^2_A.
		\]
		\item For any $-\delta A\preceq B\preceq A $, and any vectors $u$ and $v$, $a>0$
		\[
		2|\langle u,Bv\rangle |\leq a\langle u, Bu\rangle+2a\delta \|u\|_A^2+\frac{1+2\delta}{a} \|v\|_A^2. 
		\]
	\end{enumerate}
\end{lem}

\begin{proof} 
	Claim 1). Let $(l_i,u_i)$ be the eigenvalue-eigenvectors of $B$. Assume also that 
	\[
	v=\sum_{i=1}^p a_i u_i, \quad w=\sum_{i=1}^p b_i u_i.
	\] 
	Then by Young's inequality
	\begin{align*}
	2\langle v, B w\rangle&= 2\sum_{i=1}^p l_i  a_i b_i\leq a\sum_{i=1}^p |l_i|  |a_i|^2+\frac{1}{a}\sum_{i=1}^p|l_i| |b_i|^2=a\langle v,|B| v\rangle+\frac{1}{a}\langle w,|B| w\rangle.
	\end{align*}
	Next, we denote the positive part of $\Lambda$ as $\Lambda_+$ and the negative part as $\Lambda_-$, so that 
	\[
	\Lambda=\Lambda_++\Lambda_-,\quad |\Lambda|=\Lambda_+-\Lambda_-,\quad \Lambda_-\preceq 0\preceq \Lambda_+.
	\]
	Then by checking eigen-space with nonnegative eigenvalues,  $B\preceq A$ indicates that $V\Lambda_+V^T\preceq A$. Likewise, we have $-\delta A\preceq V\Lambda_-V^T $. In combination, we have 
	\[
	|B|=V\Lambda_+V^T-V\Lambda_-V^T\preceq (1+\delta) A. 
	\]
	Therefore 
	\[
	a\langle v,|B| v\rangle+\frac{1}{a}\langle w, |B| w\rangle\leq (1+\delta)a\| v\|_{A}^2+\frac{1+\delta}{a}\|w\|^2_A. 
	\]
	For claim 2), denote $B_\delta=B+\delta A\succeq 0$. Then
	\begin{align*}
	2|\langle u,Bv\rangle |&\leq  2|\langle u,B_\delta v\rangle |+ 2\delta |\langle u, Av\rangle |\\
	&\leq a \|u\|^2_{B_\delta}+ \frac{1}{a} \|v\|^2_{B_\delta}+a\delta\|u\|^2_A+\frac{\delta}{a}\|v\|^2_A\\
	&= a\langle u, B u\rangle+a\delta \|u\|^2_{A}+\frac{1}{a}\langle v, Bv\rangle+ \frac{\delta}{a} \|v\|^2_{A}+a\delta\|u\|^2_A+\frac{\delta}{a}\|v\|^2_A\\
	&\leq a\langle u, Bu\rangle+2a\delta \|u\|_A^2+\frac{1+2\delta}{a} \|v\|_A^2. 
	\end{align*}

\end{proof}
\subsection{Proof of the main results}
\begin{proof}[Proof of Theorem \ref{thm:Asgd}]
\blue{ We rewrite SGD update as  
	\begin{align}\label{equ:SGD2.3}
	w_{n+1}=w_n - \eta \nabla f_\lambda (w_n,\zeta_n) = w_{n}-\eta \nabla F_\lambda (w_{n})+ \eta \xi_n,
	\end{align}
	where
	\[
	\xi_n =\nabla F_\lambda(w_n) - \nabla f_\lambda (w_n,\zeta_n) =\nabla F(w_n) - \nabla f(w_n,\zeta_n). 
	\]
	Let $\mathcal{F}_n$ be the $\sigma$-algebra generated by $\{w_{i+1},\zeta_{i}, i=1,\ldots,n-1\}$. We use  $\E_n(\cdot)$ to denote the conditional expectation $\E(\cdot |\mathcal{F}_n)$. Then $\xi_n$ is a martingale sequence since $\E_n \xi_n \equiv 0$.  }
	
\blue{	From \eqref{equ:SGD2.3}, we find
	\begin{align}\label{equ:w2n+12.3}
	\|w_{n+1}-w^*\|^2=\|w_{n}-w^*\|^2-2\eta \langle w_n-w^*,  \nabla F_\lambda (w_{n})-\xi_n\rangle+\eta^2\| \nabla F_\lambda (w_{n})-\xi_n\|^2.
	\end{align}
	We first try to find a bound of $\langle -(w_n-w^*),  \nabla F_\lambda  (w_{n})\rangle$.  We define
	\begin{align*}
	B_n :=\int^1_0 \nabla^2 F_\lambda (sw_n+(1-s)w^*)ds=\lambda I+\int^1_0 \nabla^2 F (sw_n+(1-s)w^*)ds,
	\end{align*}
	and apply fundamental theorem of calculus on $\nabla F_{\lambda}$. Note that $\nabla F(w^{*}) = 0$, we obtain
	\begin{align}\label{tmp:F12.3}
	\nabla F_\lambda(w_n) &= \nabla F_\lambda (w^*) + \int_{0}^{1} \nabla^2 F_\lambda (sw_n+(1-s)w^* )(w_n-w^*)ds\\
	\notag
	& =\lambda w^*+B_n (w_n-w^*). 
	\end{align}
	Note that $\frac12\lambda I\preceq -\delta A+\lambda I\preceq B_n\preceq A+\lambda I$ and $B_n$ is symmetric, we have
	\begin{align}
	\notag
	\langle -(w_n-w^*),  \nabla F_\lambda  (w_{n})\rangle&=-\|w_n-w^*\|^2_{B_n}-\lambda\langle w^*,w_n-w^*\rangle\\
	\notag
	&\leq-\frac12 \|w_n-w^*\|^2_{B_n}-\frac\lambda 4 \|w_n-w^*\|^2-\lambda\langle w^*,w_n-w^*\rangle\\
	\label{ineq:2etaw_n2.3}
	&\leq-\frac12 \|w_n-w^*\|^2_{B_n}+\lambda\|w^*\|^2.
	\end{align}
	Furthermore, note that  by $B_n \preceq A+\lambda I$, we have
	\begin{align}\label{ineq:DeltaF22.3}
	\|  \nabla F_\lambda (w_{n})\|^2&=\|\lambda w^*+B_n(w_n-w^*)\|^2 \leq 2\lambda^2 \|w^*\|^2+2(\|A\|+\lambda)\|w_n-w^*\|_{B_n}^2. 
	\end{align}
	Similarly, we find 
	\begin{align*}
	\nabla F(w_n) = \int_{0}^{1} \nabla^2 F(sw_n+(1-s)w^* )ds(w_n-w^*) = (B_n-\lambda I) (w_n-w^*),
	\end{align*}
    thus
	\begin{align}\label{ineq:DeltaF22.3+1}
	\|  \nabla F (w_{n})\|^2 \leq \|A\| \|w_n-w^*\|_{B_n}^2. 
	\end{align}
	Recall that $w_n$ is $\mathcal{F}_n$-measurable and $\E_n \xi_{n} = 0$. Also 
	\[
	\E \|\xi_n\|^2\leq r^2+c_r|(w_n-w^*)^T\nabla F(w_n) |\leq r^2+c_r \|w_n-w^*\|^2_{B_n}.
	\]   }
	
\blue{	 So plugging (\ref{ineq:2etaw_n2.3}) and (\ref{ineq:DeltaF22.3}) into (\ref{equ:w2n+12.3}), using \eqref{eqn:sgvarianceless} with \eqref{ineq:DeltaF22.3+1},  and by Cauchy Schwarz inequality, we then have
	\begin{align*}
\E_n&\|w_{n+1}-w^*\|^2=\|w_{n}-w^*\|^2-2\eta \E_n\langle w_n-w^*,  \nabla F_\lambda (w_{n})-\xi_n\rangle+\eta^2\E_n\| \nabla F_\lambda (w_{n})-\xi_n\|^2\\
	&= \mathbb{E}_n[ \|w_n-w^*\|^2-2\eta \langle w_n-w^*,  \nabla F_\lambda (w_{n})\rangle+ \eta^2 (\|\nabla F_\lambda (w_n)\|^2+\|\xi_n\|^2) ]\\
	&\leq  \|w_n-w^*\|^2- \eta \|w_n-w^*\|_{B_n}^2+2\lambda\eta(1+\lambda\eta )\|w^*\|^2\\
	&+\eta^22(1+c_r)(\|A\|+\lambda)\|w_n-w^*\|_{B_n}^2 +\eta^2 r^2.
	\end{align*}
	Under the condition 
	\[
	\eta  \leq \min\Big\{ \frac1{4(1+c_r)(\|A\|+\lambda)}, 1 \Big\}, \quad \lambda \leq 1,
	\] 
	and since $0 \leq 1_{\tau\geq n+1} \leq  1_{\tau\geq n}$, we have
	\begin{align*}
	\E [1_{\tau\geq n+1}\|w_{n+1}-w^*\|^2]
	&\leq \E [ 1_{\tau\geq n}\|w_{n+1}-w^*\|^2]\\ &=\E [ 1_{\tau\geq n} \E_n\|w_{n+1}-w^*\|^2 ]\\
	&  \leq  \mathbb{E} [ 1_{\tau\geq n}(\|w_n-w^*\|^2-\tfrac12\eta \|w_n-w^*\|^2_{B_n}) ] +4\lambda\eta\|w^*\|^2+ \eta^2 r^2.
	\end{align*}
	Summing this inequality over all $n=0,\ldots, (N\wedge \tau)-1$, we find that 
	\begin{equation}
	\label{eqn:ASGDescape}
	\E [\|w_{\tau\wedge N}-w^*\|^2]\leq \E \|w_0-w^*\|^2+N(4\lambda\eta\|w^*\|^2+\eta^2 r^2).
	\end{equation}
	Summing the same inequality over all $n= 0,\ldots,N-1$, we find that  
	\begin{equation}
	\label{eqn:ASGDsum}
	\E \left[1_{\tau\geq N-1}\left(\frac12\eta \sum_{n=0}^{N-1} \|w_n-w^*\|^2_{B_n}\right)\right]\leq \E \|w_0-w^*\|^2 +N(4\lambda\eta\|w^*\|^2+\eta^2 r^2). 
	\end{equation} 
}	
\blue{	To continue, recall $G(w_n)=F(w_n)-F(w^*)$.  Apply fundamental theorem of calculus to $F(w)$, we obtain
	\begin{align}\label{tmp:F-Fl2.3}
	\begin{split}
	G(w_n) &= F(w_n) - F(w^{*})\\
	&= \Big[ \int_{0}^{1} \nabla F(sw_n + (1-s)w^{*}) ds \Big]^{T} (w_n - w^{*}) \\
	& = \Big[ \int_{0}^{1} \Big ( \nabla F(w^{*}) + \int_{0}^{s} \nabla^2F(tw_n + (1-t)w^{*})(w_n - w^{*}) dt\Big)ds \Big]^{T} (w_n - w^{*}) \\
	& = (w_n - w^{*})^{T} \Big[ \int_{0}^{1}(1-s) \nabla^2F(sw_n + (1-s)w^{*}) ds \Big] (w_n - w^{*})\\
	&=  (w_n-w^*)^TA_n (w_n-w^*),
	\end{split}
	\end{align}
	with 
	\begin{align*}
	A_n = \int^1_0 (1-s)\nabla^2 F(sw_n+(1-s)w^*)ds.
	\end{align*}
	Under Assumption \ref{aspt:sparse}, 
	we observe that 
	\begin{align*}
 	A_n+\frac12\lambda I& =\int_{0}^{1}(1-s) \nabla^2F_\lambda(sw_n + (1-s)w^{*}) ds\\
	&\preceq  \int_{0}^{1}\nabla^2F_\lambda (sw_n + (1-s)w^{*}) ds =B_n.
	\end{align*}
	Namely, we have $H(w_n)=G(w_n)+\frac{\lambda}2\|w_n-w^*\|^2\leq \|w_n-w^*\|^2_{B_n}$. Together with \eqref{eqn:ASGDsum}, we obtain
	\[
	\E \left[1_{\tau\geq N-1}\left(\frac12\eta \sum_{n=0}^{N-1} H(w_n)\right)\right]\leq \E \|w_0-w^*\|^2 +N(4\lambda\eta\|w^*\|^2+\eta^2 r^2). 
	\] 
     Then because $H$ is convex within $\calD$, we have $H(\bar{w}_N)\leq \frac{1}{N}\sum_{n=0}^{N-1} H(w_n)$ and 
	\[
	\E \left[1_{\tau\geq N-1}\left(\eta N H(\bar{w}_N)\right)\right]\leq 2\E \|w_0-w^*\|^2 +2N(4\lambda\eta\|w^*\|^2+\eta^2 r^2).
	\]
	This leads to our claim
	\[
	\E \left[1_{\tau\geq N-1} G(\bar{w}_N)\right]\leq \E \left[1_{\tau\geq N-1} H(\bar{w}_N)\right]\leq \frac{2\E \|w_0-w^*\|^2}{N\eta} +8\lambda\|w^*\|^2+2\eta r^2. 
	\]	}
\end{proof}

\begin{proof}[Proof of Theorem \ref{thm:sgd}]
	\textbf{Step 1:} we build a bound for $\|w_n\|^2$. We rewrite SGD update as  
	\begin{align}\label{equ:SGD}
	w_{n+1}=w_n - \eta \nabla f_\lambda(w_n,\zeta_n) = w_{n}-\eta \nabla F_\lambda (w_{n})+ \eta \xi_n,
	\end{align}
	where
	\[
	\xi_n =\nabla F_\lambda(w_n) - \nabla f_\lambda(w_n,\zeta_n)=\nabla F(w_n) - \nabla f(w_n,\zeta_n). 
	\]
	Let $\mathcal{F}_n$ be the $\sigma$-algebra generated by $\{w_{i+1},\zeta_{i}, i=1,\ldots,n-1\}$. We use  $\E_n(\cdot)$ to denote the conditional expectation $\E(\cdot |\mathcal{F}_n)$. Then $\xi_n$ is a martingale sequence since $\E_n \xi_n \equiv 0$. 
	
	From \eqref{equ:SGD}, we find
	\begin{align}\label{equ:w2n+1}
	\|w_{n+1}\|^2=\|w_{n}\|^2-2\eta \langle w_n,  \nabla F_\lambda (w_{n})-\xi_n\rangle+\eta^2\| \nabla F_\lambda (w_{n})-\xi_n\|^2.
	\end{align}
	To continue, we try to find a bound of $-2\eta \langle w_n,  \nabla F_\lambda (w_{n})\rangle$.  We define
	\begin{align*}
	B_n :=\int^1_0 \nabla^2 F(sw_n+(1-s)w^*)ds,
	\end{align*}
	and apply fundamental theorem of calculus on $\nabla F$. Note that $\nabla F(w^{*}) = 0$, we obtain
	\begin{align}\label{tmp:F1}
	\nabla F(w_n) = \nabla F(w^{*}) + \int_{0}^{1} \nabla^2 F(sw_n + (1-s)w^{*})(w_n - w^{*})ds = B_n(w_n - w^{*}).
	\end{align}
	Note that $-\delta A\preceq B_n\preceq A$.  We have
	\begin{align*}
	&\langle -w_n,  \nabla F (w_{n})\rangle=-\langle w_n,B_n (w_n-w^*)\rangle\\
	&=-\langle w_n,(B_n+\delta A) w_n\rangle+\langle w_n, (B_n+\delta A) w^*\rangle+\delta (\|w_n\|^2_A-\langle w_n,A w^*\rangle)\\
	&\leq \frac14\|w^*\|^2_{B_n+\delta A}+\delta (2\|w_n\|^2_A+\frac14\|w^*\|^2_A) \\
	&\leq2\delta \|A\|\|w_n\|^2+\frac12\|w^*\|^2_A \mbox{ since   $\delta \leq \frac12$  and $\|w_n\|_A^2 \leq \|A\|\|w_n\|^2$.}
	\end{align*}
	Recall that  $\delta \|A\|\leq \frac{\lambda}{4}$, we find
	\begin{align}\label{ineq:2etaw_n}
	\begin{split}
	-2\eta \langle w_n,  \nabla F_\lambda (w_{n})\rangle &= -2\eta \langle w_n,  \nabla F(w_{n}) + \lambda w_n \rangle \\
	&=-2\lambda \eta\|w_n\|^2+2\eta\langle -w_n,  \nabla F (w_{n})\rangle\\
	&\leq -\lambda\eta \|w_n\|^2+ \eta \|w^*\|_{A}^2. 
	\end{split}
	\end{align}
If $B_n=Q\Lambda Q^T$ is the eigendecomposition of $B_n$, let $|B_n|=Q|\Lambda|Q^T$, where $|\Lambda|$ takes absolute value on each element of the diagonal matrix $\Lambda$. From the proof of Lemma \ref{lem:norm} claim 1), we know $|B_n|\preceq (1+\delta) A\preceq (1+\delta)\|A\|I$. Thus  by $B_n \preceq A$, we have
	\begin{align}
	\notag
	\|  \nabla F (w_{n})\|^2&=\| B_n(w^*-w_n)\|^2 \leq 2\|B_n w^*\|^2+2\|B_nw_n\|^2\\
	\notag
	&\leq 2 (w^*)^TB_n^{1/2}|B_n| B_n^{1/2} w^*+2\|A\|^2\|w_n\|^2\\
	\notag
	&\leq 2(1+\delta)\|A\|\|B_n^{1/2} w^*\|^2+2\|A\|^2\|w_n\|^2 \\
	\label{ineq:DeltaF2}
	&\leq 4\|A\|\|w^*\|_{A}^2+2\|A\|^2\|w_n\|^2. 
	\end{align}
	Recall that $w_n$ is $\mathcal{F}_n$-measurable and $\E_n \xi_{n} = 0$, we have  $\E_n\langle w_n,  \xi_n \rangle = 0$. So plugging (\ref{ineq:2etaw_n}) and (\ref{ineq:DeltaF2}) into (\ref{equ:w2n+1}), and by Cauchy Schwarz inequality, we then have
	\begin{align*}
	&\mathbb{E}_n \|w_{n+1}\|^2 \\
	& =\mathbb{E}_n[ \|w_n\|^2-2\eta \langle w_n,  \nabla F_\lambda (w_{n})\rangle+ \eta^2\| \nabla F (w_{n}) + \lambda w_n-\xi_n\|^2 ]\\
	&\leq \mathbb{E}_n[ \|w_n\|^2-2\eta \langle w_n,  \nabla F_\lambda (w_{n})\rangle+ 3\eta^2 (\|\nabla F(w_n)\|^2+\|\xi_n\|^2+\lambda^2 \|w_n\|^2) ]\\
	&\leq  \|w_n\|^2-\lambda \eta \|w_n\|^2 +6\eta^2 \|A\|^2\|w_n\|^2+3\eta^2\lambda^2 \|w_n\|^2 \\
	& \quad +   
	\left( \eta  + 12\|A\|\eta^2 \right)   \|w^*\|_{A}^2  +  3\eta^2r^2(1+c_r\|w_n\|^2).
	\end{align*}
	Under the condition 
	\[
	\eta  \leq \dfrac{\lambda}{12\|A\|^2+6 \lambda^2+6c_rr^2},
	\] 
\blue{	we have that if $\tau\geq n$
	\begin{align*}
	\mathbb{E}_n \|w_{n+1}\|^2 
	\leq  (1-\tfrac12\lambda \eta )\|w_n\|^2 +	\eta M_w,
	\end{align*}
	which can also leads to
	\begin{align*}
	& \mathbb{E}_n \|w_{n+1}\|^21_{\tau\geq n+1}\leq  \mathbb{E}_n \|w_{n+1}\|^21_{\tau\geq n}
	\leq  (1-\tfrac12\lambda \eta )\|w_n\|^21_{\tau \geq n} +\eta M_w, \\
	& \mathbb{E}_n \|w_{n+1}\|^21_{\tau\geq n+1} - \|w_n\|^21_{\tau \geq n}
	\leq \eta M_w,
	\end{align*} }
	where
	\begin{align*}
	M_w := \left( 1 + 12\|A\|\eta \right) \|w^*\|_{A}^2   +3\eta r^2 =  \|w^*\|_{A}^2   + \eta\left( 12\|A\|\|w^*\|_{A}^2 + 3r^2  \right).
	\end{align*}
Then iterating the inequalities above gives us 
	\begin{align}\label{eqn:wnl2}
	\E [\|w_n\|^21_{\tau\geq n}] &\leq \left(1-\frac{\lambda \eta}{2}\right)^n\E \|w_0\|^2+ \frac{2}{\lambda} M_w, \\
\blue{ \E [\|w_{n\wedge \tau}\|^2] }& \blue{ \leq \E \|w_0\|^2+ \eta n M_w. }
	\end{align}

%
%
%
	
	\noindent\textbf{Step 2:}  we derive how does the  generalization error evolve. According to Taylor's expansion and (\ref{equ:SGD}), we know that there exists a $v_n$, such that 
	\begin{align}
	\notag
	F(w_{n+1})&=F(w_{n})-\eta \|\nabla F(w_{n})\|^2-\eta \lambda \nabla F(w_n)^Tw_n  + \eta \xi_n^T \nabla F(w_{n})   \\
	\notag
	&\quad +\frac{\eta^2}2 (\nabla F(w_{n})+\lambda w_n-\xi_n)^T \nabla^2 F(v_n) (\nabla F(w_{n})+\lambda w_n-\xi_n)\\
	\notag
	&\leq F(w_{n})-\eta \|\nabla F(w_{n})\|^2-\eta \lambda \nabla F(w_n)^Tw_n +  \eta \xi_n^T \nabla F(w_{n})   \\
	\notag
	&\quad +\frac{\eta^2}2 (\nabla F(w_{n})+\lambda w_n-\xi_n)^TA (\nabla F(w_{n})+\lambda w_n-\xi_n)\\
	\notag
	&\leq F(w_{n})-\eta \|\nabla F(w_{n})\|^2-\eta \lambda \nabla F(w_n)^T(w_n-w^*) + \eta\lambda|\nabla F(w_n)^Tw^*|\\
	\label{equ:Fw_n+1}
	&\quad + \eta \xi_n^T \nabla F(w_{n}) +\frac{3\eta^2}{2}\|A\|(\|\nabla F(w_{n})\|^2+ \lambda^2\|w_n\|^2+\|\xi_n\|^2 ). 
	\end{align}
	\noindent\textbf{Step 3:} we bound each term in \eqref{equ:Fw_n+1} through interpolation. We observe that 
	\begin{align}\label{lambda}
	|\lambda \nabla F(w_n)^Tw^*|\leq  |\lambda \nabla F(w_n)^Tw_\bot^*|+ |\lambda \nabla F(w_n)^Tw_\lambda^*|, 
	\end{align}
	where $w^*=w_\bot^*+w_\lambda^*$ is the decomposition introduced in Definition \ref{defn}. Next
	\begin{align}\label{lambda1}
	|\lambda \nabla F(w_n)^Tw_\lambda^*|\leq   \frac12\|\nabla F(w_n)\|^2+\frac12\lambda^2 \|w_\lambda^*\|^2.
	\end{align}
	Recall $\nabla F(w_n) = B_n (w_n-w^*)$ in (\ref{tmp:F1}), and  by Lemma \ref{lem:norm} claim 2), we further have 
	\begin{align}\label{lambda2}
	\begin{split}
	|\lambda \nabla F(w_n)^Tw_\bot^*|& \leq  \frac12\lambda(w_n-w_*)^TB_n(w_n-w_*) \\
	& \quad + \delta\lambda \|w_n-w_*\|_A^2 + \frac{1+2\delta}{2}\lambda w_\bot^{*T}A w^*_\bot.
	\end{split}
	\end{align}
	Plugging (\ref{lambda1}), (\ref{lambda2}) into (\ref{lambda}), applying the result to (\ref{equ:Fw_n+1}) gives us
	\begin{align}
\notag
	F(w_{n+1})&\leq F(w_{n})-\frac12\eta \|\nabla F(w_{n})\|^2- \frac12\lambda \eta \nabla F(w_n)^T(w_n-w_*)+ \frac{1+2\delta}{2}\lambda \eta\|w^*\|^2_{A,\lambda} \\
	\notag
	&\quad  + \delta\lambda\eta \|w_n - w^*\|_{A}^2 +\eta \xi_n^T \nabla F(w_{n})  + \frac{3\eta^2}{2}\|A\|(\|\nabla F(w_{n})\|^2+ \lambda^2\|w_n\|^2+\|\xi_n\|^2 )\\
	\notag
	&\mbox{(Recall that $\eta \leq 1$, $\delta < \frac{1}{2}$ and  $\eta/2 - 3\eta^2\|A\|/2 \geq 0$) }\\
		\label{tmp:Fn}
	&\leq F(w_{n})- \frac12\lambda \eta \nabla F(w_n)^T(w_n-w_*)+ \lambda \eta\|w^*\|^2_{A,\lambda} \\
	\notag
	&\quad  + \delta\lambda\eta \|w_n - w^*\|_{A}^2 +\eta \xi_n^T \nabla F(w_{n})  + \frac{3\eta^2}{2}\|A\|(\lambda^2\|w_n\|^2+\|\xi_n\|^2 ).
	\end{align}
	To continue, recall $G(w_n)=F(w_n)-F(w^*)$.  Apply fundamental theorem of calculus to $F(w)$, we obtain
	\begin{align}
	\notag
	G(w_n) &= F(w_n) - F(w^{*})\\
	\notag
	&= \Big[ \int_{0}^{1} \nabla F(sw_n + (1-s)w^{*}) ds \Big]^{T} (w_n - w^{*}) \\
	\notag
	& = \Big[ \int_{0}^{1} \Big ( \nabla F(w^{*}) + \int_{0}^{s} \nabla^2F(tw_n + (1-t)w^{*})(w_n - w^{*}) dt\Big)ds \Big]^{T} (w_n - w^{*}) \\
	\notag
	& = (w_n - w^{*})^{T} \Big[ \int_{0}^{1}(1-s) \nabla^2F(sw_n + (1-s)w^{*}) ds \Big] (w_n - w^{*})\\
	\label{tmp:F-Fl}
	&=  (w_n-w^*)^TA_n (w_n-w^*),
	\end{align}
	with 
	\begin{align*}
	A_n = \int^1_0 (1-s)\nabla^2 F(sw_n+(1-s)w^*)ds.
	\end{align*}
	Under Assumption \ref{aspt:sparse}, namely $0 \preceq \nabla^2 F(w_n)+\delta A\preceq A+\delta A $, 
	we observe that 
	\begin{align*}
	\frac12\delta A + A_n& =\frac12\delta A+\int_{0}^{1}(1-s) \nabla^2F(sw_n + (1-s)w^{*}) ds\\
	&=\int_{0}^{1}(1-s) (\nabla^2F(sw_n + (1-s)w^{*})+\delta A) ds \\
	&\preceq  \int_{0}^{1}(\nabla^2F(sw_n + (1-s)w^{*})+\delta A) ds\\
	& =B_n +\delta A \preceq (1+\delta)A. 
	\end{align*}
Namely, we have
	\[
	A_n\preceq B_n+\frac12 \delta A\preceq (1+\tfrac{\delta}{2} )A.
	\]
	Thus 
	\[
	G(w_n)-\frac12\delta \|w_n-w^*\|^2_A\leq  (w_n-w^*)^TB_n (w_n-w^*)=\nabla F(w_n)^T(w_n-w_*). 
	\] 
	Plug this into \eqref{tmp:Fn}, together with $\|w_n-w^*\|^2_A  \leq 2\|w_n\|^2_A + 2\|w^*\|^2_A  \leq 2\|A\|\|w_n\|^2 + 2\|w^*\|^2_A $,  we have
	\begin{align*}
	G(w_{n+1})&\leq G(w_{n})- \frac12 \eta \lambda  G(w_n)+\eta \xi_n^T \nabla F(w_{n}) \\
	&\quad+ \frac{5}{4}\delta \lambda \eta \|w_n-w^*\|^2_A+  \lambda \eta  \|w^*\|^2_{A,\lambda} + \frac{3\eta^2}{2}\|A\|( \lambda^2\|w_n\|^2+\|\xi_n\|^2 ).\\
	&\leq G(w_{n})-  \frac12\eta \lambda  G(w_n)+\eta \xi_n^T \nabla F(w_{n})\\
	&\quad+  \lambda \eta ( \|w^*\|^2_{A,\lambda} + \frac{5}{2} \delta \|w^*\|^2_A)  + \left(\frac{3\eta^2\lambda^2}{2}+\frac{5}{2}\lambda\eta\delta \right)\|A\| \|w_n\|^2+\frac{3\eta^2}{2}\|A\|\|\xi_n\|^2.
	\end{align*}
	\textbf{Step 4:} summarizing arguments. \blue{ We will first establish a rough estimate, which is useful to the escape probability.  Since $ \E_n\xi_n^T \nabla F(w_{n}) \equiv 0$, we have 
	\begin{align*}
	\notag
	\E_{n\wedge \tau}[ G(w_{n\wedge \tau+1})] 	&\leq (1- \frac12 \eta\lambda  ) G(w_{n\wedge \tau})  +  \left(\frac{3\eta^2\lambda^2}{2}+  \frac{5}{2}\lambda\eta\delta \right)\|A\|  \|w_{n\wedge \tau}\|^2\\
		\notag
	&\quad +  \lambda \eta \left( \|w^*\|^2_{A,\lambda} + \frac52\delta \|w^*\|^2_A\right)  +\frac{3\eta^2\|A\|}2r^2(1+c_r G(w_{n\wedge\tau})] )\\
		\notag
	&\mbox{Because $\eta\leq \frac{\lambda}{6c_r\|A\|r^2}$} \\
		\notag
	&\leq G(w_{n\wedge\tau})  +  \left(\frac{3\eta^2\lambda^2}{2}+  \frac{5}{2}\lambda\eta\delta \right)\|A\|  \|w_{n\wedge \tau}\|^2 \\
	&\quad +  \lambda \eta \left( \|w^*\|^2_{A,\lambda} + \frac52\delta \|w^*\|^2_A\right)  +\frac{3\eta^2\|A\|}2r^2. 
	\end{align*}
	Recall that $\E [\|w_{n\wedge \tau}\|^2] \leq \E \|w_0\|^2+ \eta n M_w$ with
	\[
	M_w=\|w^*\|_{A}^2   + \eta\left( 12\|A\|\|w^*\|_{A}^2 + 3r^2  \right).
	\]
Iterating above result gives us 
	\begin{align}
	\label{eqn:SGDescape}
	\E [ G(w_{n\wedge \tau}) ] \leq &\E [ G(w_0) ]+\left(\frac{3\eta^2\lambda^2}{2}+  \frac{5}{2}\lambda\eta\delta \right)\|A\|  (n\E\|w_0\|^2+n^2\eta M_w )  \\
	\notag
	&+\lambda n \eta \left( \|w^*\|^2_{A,\lambda} + \frac52\delta \|w^*\|^2_A\right)  +\frac{3n\eta^2\|A\|}2r^2.
	\end{align} }
We can further improve this bound by   using $0 \leq 1_{\tau\geq n+1} \leq  1_{\tau\geq n}\leq 1$ and taking conditional expectation for both sides. Since $ \E_n\xi_n^T \nabla F(w_{n}) \equiv 0$, we have 
\blue{	\begin{align}
	\notag
	\E[ G(w_{n+1})1_{\tau\geq n+1} ] &\leq \E[ G(w_{n+1})1_{\tau\geq n} ] =  \E[ 1_{\tau\geq n} \E_n[G(w_{n+1})]]\\
		\notag
	&\leq (1- \frac12 \eta\lambda  ) \mathbb{E}[ G(w_n) 1_{\tau\geq n} ] +  \left(\frac{3\eta^2\lambda^2}{2}+  \frac{5}{2}\lambda\eta\delta \right)\|A\| \mathbb{E}[ \|w_n\|^2 1_{\tau\geq n}]\\
		\notag
	&\quad +  \lambda \eta \left( \|w^*\|^2_{A,\lambda} + \frac52\delta \|w^*\|^2_A\right)  +\frac{3\eta^2\|A\|}2r^2(1+c_r\E [G(w_n)1_{\tau\geq n}] )\\
		\notag
	&\mbox{Because $\eta\leq \frac{\lambda}{6c_r\|A\|r^2}$} \\
		\notag
	&\leq (1- \frac14 \eta\lambda  ) \mathbb{E}[ G(w_n) 1_{\tau\geq n} ] +  \left(\frac{3\eta^2\lambda^2}{2}+  \frac{5}{2}\lambda\eta\delta \right)\|A\| \mathbb{E}[ \|w_n\|^2 1_{\tau\geq n}]\\
		\label{tmp:SGD1step}
	&\quad +  \lambda \eta \left( \|w^*\|^2_{A,\lambda} + \frac52\delta \|w^*\|^2_A\right)  +\frac{3\eta^2\|A\|}2r^2. 
	\end{align} }
	Since $\eta \lambda \leq 1$, we have $0 \leq 1-\frac14\lambda \eta \leq \exp(-\frac14\lambda \eta )$, then iterating above result gives us 
	\begin{align*}
	\E [ G(w_n)1_{\tau\geq n} ] \leq &\exp(- \frac{1}{4} \lambda n \eta   )\E [ G(w_0) ]+4 \|w^*\|^2_{A,\lambda} + 10 \delta\|w^*\|^2_A+  \frac{6\eta\|A\|}{\lambda}r^2  \\
	&+ \left(\frac{3\eta^2\lambda^2}{2}+  \frac{5}{2}\lambda\eta\delta\right)\|A\|\sum_{i=0}^n(1-  \frac14 \eta\lambda  )^{n-i}\E [ \|w_i\|^21_{\tau\geq i} ]. 
	\end{align*}
	Applying (\ref{eqn:wnl2}), together with $\lambda \leq 1, \eta \leq 1$, $\delta \leq \frac{1}{2}$, $12\eta\|A\|\leq 1$ and $ 1-\frac14\lambda \eta \leq \exp(-\frac14\lambda \eta ) $, we obtain
	\begin{align*}
	& \E [ G(w_n)1_{\tau\geq n} ] \\
	& \leq \exp(- \frac{1}{4} \lambda n \eta  )\E [ G(w_0) ]+ 4 \|w^*\|^2_{A,\lambda} +  10 \delta\|w^*\|^2_A+\frac{ 6 \eta\|A\|}{\lambda}r^2\\
	&\quad +\left(\frac{3\eta^2\lambda^2}{2} + \frac{5}{2}\lambda\eta \delta \right)\|A\| \sum_{i=0}^n\left( (1-  \frac{1}{4} \lambda \eta)^n\mathbb{E}[ \|w_0\|^2 ] + (1-\frac14 \lambda\eta )^{n-i}\frac{2}{\lambda}M_w \right) \\
	& \leq  \exp(-\frac{1}{4}  \lambda n \eta  )\E [ G(w_0)+ 4n\|A\|\|w_0\|^2 ]+ 4\|w^*\|^2_{A,\lambda} +10\delta\|w^*\|^2_A+ \frac{6\eta\|A\|}{\lambda}r^2\\
	&\quad+  \frac{(12\lambda\eta + 20\delta)\|A\|}{\lambda} \left(
	  \|w^*\|_{A}^2  + \eta\left( 12\|A\|\|w^*\|_{A}^2 + 3r^2  \right) \right) \\
	  & \leq  \exp(-\frac{1}{4}  \lambda n \eta  )\E [ G(w_0)+ 4n\|A\|\|w_0\|^2 ]+ 4\|w^*\|^2_{A,\lambda} +10\delta\|w^*\|^2_A+ \frac{6\eta\|A\|}{\lambda}r^2\\
	&\blue{ \quad+  \frac{(12\lambda\eta + 20\delta)\|A\|}{\lambda} \left(2
	  \|w^*\|_{A}^2  +  3r^2\eta  \right) } \\
	&
\blue{	\leq 4 \|w^*\|^2_{A,\lambda}  +\frac{C_1}{\lambda} (\eta+\delta)+\exp(-\frac14 \lambda n \eta  )\E [ G(w_0)+ 4n\|A\|\|w_0\|^2 ], }
	\end{align*}
	with \blue{ $C_1=60\|A\| \left(r^2+\|w^*\|_A^2\right) + 10\|w^*\|_A^2$}. 
	
\end{proof}

\begin{proof}[Proof of Corollary \ref{cor:selectpara_onlineridgeless}]
\blue{	By Theorem \ref{thm:Asgd},  $\E [ G(\bar{w}_N)1_{\tau\geq N} ] \leq 3\epsilon$ holds if we choose $\lambda, \eta, N, \delta$ such that the following results hold
	\begin{align*}
	\frac{2\E \|w_0-w^*\|^2}{N\eta} \leq \epsilon, \  8\lambda\|w^*\|^2 \leq \epsilon, \ 2\eta r^2 \leq \epsilon, 
	\end{align*}
	and the following conditions are satisfied
	\begin{align*}
	\eta \leq \Big\{ \frac1{2(1+c_r)(\|A\|+\lambda)}, 1\Big\}, \quad  2\delta \|A\| \leq \lambda \leq 1.
	\end{align*}
	Solving $8\lambda\|w^*\|^2 \leq 8\lambda C_0 \leq  \epsilon$ gives us $\lambda(\epsilon) \leq \frac{\epsilon}{8C_0}$. The condition on $\delta(\epsilon)$ is obtained from $\lambda \geq 2\delta \|A\|$. 
	The condition of $\eta(\epsilon)$ ensures that $2\eta r^2 \leq \epsilon$ and $\eta \leq \frac1{2(1+c_r)(\|A\|+\lambda)}$. 
	With chosen $\eta(\epsilon)$, the scale of $N(\epsilon)$ is obtained by solving $\frac{2\E \|w_0-w^*\|^2}{N\eta} \leq \frac{2C_0}{N\eta} \leq \epsilon$. }
\end{proof}

\begin{proof}[Proof of Corollary \ref{cor:selectpara_online}]
	By Theorem \ref{thm:sgd},  $\E [ G(w_N)1_{\tau\geq N} ] \leq 4\epsilon$ holds if we choose $\lambda, \eta, N, \delta$ such that the following results hold
	\begin{align*}
	& 4\|w^*\|^2_{A,\lambda} \leq \epsilon, \   \frac{C_1\eta}{\lambda} \leq \epsilon , \ \blue{ \frac{C_1\delta}{\lambda} \leq \epsilon},\ \ \exp(-\frac14 \lambda N \eta)\E [ G(w_0)+ 4N\|A\|\|w_0\|^2 ] \leq \epsilon.
	\end{align*}
	We first choose $\lambda(\epsilon)$ such that $4 \| w^*\|_{A,\lambda(\epsilon)}^2< \epsilon$. \blue{ The conditions on $\eta(\epsilon), \delta(\epsilon)$ ensure that $\frac{C_1\eta}{\lambda} \leq \epsilon$ and $\frac{C_1\delta}{\lambda} \leq \epsilon$}. 
	With chosen $\lambda(\epsilon), \eta(\epsilon)$, the scale of $N(\epsilon)$ is obtained by solving $\exp(-\frac14 \lambda N \eta  )\E [ G(w_0)] \leq \frac{\epsilon}{2}$ and $\exp(-\frac14 \lambda N \eta  )4N\|A\| \E [ \|w_0\|^2 ] \leq \frac{\epsilon}{2}$ by using
	\[
	\exp(-\frac14 \lambda N \eta)N = \frac{4}{\lambda \eta} \exp(-\frac14 \lambda N \eta) \frac{1}{4}N\lambda \eta\leq \frac{8}{\lambda \eta} \exp(-\frac18 \lambda N \eta),
	\]
	which is derived from $\exp(-x)x\leq 2\exp(-\frac12x)$, since by Taylor expansion $x\leq 2\exp(\frac{x}{2})$.
\end{proof}

\section{Proof for Results in Low Effective Dimension in Section \ref{sec:low_eff}}

\begin{proof}[Proof of Proposition \ref{prop:sparse}]
\blue{	Under Assumption \ref{aspt:sparse}, applying Corollary \ref{cor:selectpara_onlineridgeless} results in
\begin{equation*}
\lambda(\epsilon) = O(\epsilon ), \ \delta(\epsilon) = O(\epsilon), \ \eta(\epsilon)  = O(\epsilon), \ N(\epsilon) = \Omega\Big(\frac{1}{\epsilon^2}\Big),
\end{equation*}
	 for guaranteeing $\E [ G(\bar{w}_N)1_{\tau\geq N} ] \leq 3\epsilon$. In this case, by \eqref{eqn:ASGDescape}, according to Chebyshev's inequality and recall that 
	 \[
	 \E \|w_{N\wedge \tau}-w^*\|^2\leq \E\|w_0-w^*\|^2+N(4\lambda\eta\|w^*\|^2+\eta^2 r^2) = O(1),
	 \]
	 we have 
	 \begin{align*}
	&\Prob(\tau\leq N) = \Prob(\{w_{N\wedge \tau} \notin \calD\}) = \Prob(\{ \|w_{N\wedge \tau}-w^*\|^2> A \})\leq  \frac{\E \|w_{N\wedge \tau}-w^*\|^2}{A} \leq  \delta. 
	 \end{align*}
}	 
\blue{ If $\delta=0$, according to Theorem \ref{thm:Asgd}, 
	 \begin{align*}
	 \E \left[1_{\tau\geq N-1} G(\bar{w}_N)\right]\leq \frac{2\E \|w_0-w^*\|^2}{N\eta} +8\lambda\|w^*\|^2+2\eta r^2,
	 \end{align*}
	 taking
	 \begin{align*}
	 \lambda(\epsilon) = 0, \ \eta(\epsilon)  = O(\epsilon^{1+\alpha}), \ N(\epsilon) = \Omega\Big(\frac{1}{\epsilon^{2+\alpha}}\Big),
	 \end{align*}
	 we obtain $\E [ G(\bar{w}_N)1_{\tau\geq N} ] \leq 3\epsilon$.
	 In this case, according to Chebyshev's inequality and recall that 
	 \[
	 \E \|w_{N\wedge \tau}-w^*\|^2\leq \E\|w_0-w^*\|^2+N(4\lambda\eta\|w^*\|^2+\eta^2 r^2) =\E\|w_0-w^*\|^2 +  O(\epsilon^{\alpha}),
	 \]
	 we have 
	 \begin{align*}
	 &\Prob(\tau\leq N) = \Prob(\{w_{N\wedge \tau} \notin \calD\}) = \Prob(\{ \|w_{N\wedge \tau}-w^*\|^2> A \})\leq  \frac{\E \|w_{0}-w^*\|^2+  O(\epsilon^{\alpha})}{A} < 1,
	 \end{align*}
	 if $A>\E\|w_0-w^*\|^2$.  }
\end{proof}

\begin{proof}[Proof of Proposition \ref{prop:spetrum}]
	Recall that $(\lambda_i, v_i)$, for $i=1,...,p$, are the eigenvalue-eigenvectors of $A$ with $\lambda_i$ decreasingly sorted.
	Therefore, we have
	\begin{align*}
	\|w^*\|^2_A=w^{*T}Aw^* = \sum_{i=1}^p\lambda_i \langle w^*, v_i\rangle^2 \leq \|w^*\|^2_{A,S} \tr(A),\\
	\|w^*\|^2_{A,\lambda}=\sum_{i=1}^p\lambda_i\wedge \lambda \langle w^*, v_i\rangle^2 \leq \|w^*\|^2_{A,S} \sum_{i=1}^p\lambda_i\wedge \lambda.
	\end{align*}
	For an exponential spectrum, given any $k$ and $p$, 
	\begin{align*}
	\sum_{i= k+1}^p \lambda_i = \sum_{i= k+1}^p e^{-ci} = \frac{e^{-(k+1)c}(1-e^{(k-p)c})}{1-e^{-c}} \leq \frac{1}{e^{kc}(e^c-1)}.
	\end{align*}
	Thus, to make $\sum_{i= k+1}^p \lambda_i \leq \frac{\epsilon}{8}\|w^*\|^2_{A,S}$, it is sufficient for us to take $k \geq \frac{1}{c}\log\big\{\frac{8 \|w^*\|^2_{A,S}}{\epsilon(e^c-1)}\big\} $. And to make $k\lambda=\frac{\epsilon}{8 \|w^*\|^2_{A,S}},$ we take $\lambda=\frac{\epsilon}{8 k\|w^*\|^2_{A,S}}=\tildeO(\frac{\epsilon}{|\log \epsilon|})$. By these choices,  we have
	\[
	\|w^*\|^2_{A,\lambda}\leq \sum_{i= 1}^p \lambda\wedge \lambda_i \|w^*\|^2_{A,S}\leq \frac{\epsilon}{4}. 
	\]
	Next, we find that $\|A\|\leq \tr(A)=\tildeO(1)$, so $C_1=\tildeO(1), C_2=\tildeO(1)$. We implement Corollary \ref{cor:selectpara_online} and find 
	\[
	\delta(\epsilon)=\tildeO\left(\frac{\epsilon^3}{|\log \epsilon|^2}\right),\quad \eta(\epsilon)=\tildeO\left(\frac{\epsilon^2}{|\log \epsilon|}\right),\quad N(\epsilon)=\widetilde{\Omega}\left(\frac{|\log \epsilon|^3}{\epsilon^3}\right).
	\]
\blue{	Recall that
	\begin{align*}
	\E [ G(w_{n\wedge \tau}) ] \leq &\E [ G(w_0) ]+\left(\frac{3\eta^2\lambda^2}{2}+  \frac{5}{2}\lambda\eta\delta \right) n^2\eta C_1   \\
	&+\lambda n \eta \left( \|w^*\|^2_{A,\lambda} + \frac52\delta \|w^*\|^2_A\right)  +\frac{3n\eta^2\|A\|}2r^2,
	\end{align*}
	together with
	\begin{align*}
	&\left(\frac{3\eta^2\lambda^2}{2}+  \frac{5}{2}\lambda\eta\delta \right)N^2\eta=\tildeO(\epsilon^2 |\log\epsilon|), \ \lambda N\eta\left( \|w^*\|^2_{A,\lambda} + \frac52\delta \|w^*\|^2_A\right) =\tildeO(\epsilon |\log \epsilon|),\\
	& N\eta^2 \frac{3\|A\|r^2}2=\tildeO(\epsilon |\log \epsilon|),
	\end{align*}
	we have 
	\begin{align*}
	\E [ G(w_{N\wedge \tau}) ] \leq &\E [ G(w_0) ]+ \tildeO(\epsilon |\log \epsilon| ). 
	\end{align*} }

	For a polynomial spectrum, the derivation is similar. Given any $k$ and $p$
	\begin{align*}
	\sum_{i= k+1}^p \lambda_i = \sum_{i= k+1}^p i^{-(1+c)}  \leq \sum_{i= k+1}^{p} \int_{i-1}^{i} \frac{1}{x^{1+c}}dx = \sum_{i= k+1}^{p} \frac{-1}{c}x^{-c}\Big|_{i-1}^{i} = \frac{1}{c} (k^{-c} - p^{-c}) \leq \frac{1}{ck^{c}}.
	\end{align*}
	Thus to make $\|w^*\|^2_{A,S}\sum_{i= k+1}^p \lambda_i\leq \frac{1}{8 } \epsilon$, we take $k\geq (\frac{8 \|w^*\|^2_{A,S}}{c\epsilon})^{1/c} $.  Next, we take  $\lambda(\epsilon)=\frac{\epsilon}{8 \|w^*\|^2_{A,S}k}= \tildeO\Big( \epsilon^{\frac{c+1}{c}} \Big)$. This leads to $\|w^*\|^2_{A,\lambda}\leq \epsilon/4$. Again we find that $C_1=\tildeO(1), C_2=\tildeO(1)$. The order of $\delta(\epsilon), \eta(\epsilon)$ and $N(\epsilon)$ can be derived by Corollary \ref{cor:selectpara_online}, that is,
\blue{		\[
	 \delta(\epsilon) = \tildeO\Big( \epsilon^{\frac{3c+2}{c}} \Big), \ \eta(\epsilon)  = \tildeO\Big( \epsilon^{\frac{2c+1}{c}} \Big), \ N(\epsilon) = \widetilde{\Omega}\Big(\frac{|\log(\epsilon)|}{\epsilon^{\frac{3c+2}{c}}}\Big).
	\]
	Recall that 
	\begin{align*}
	\E [ G(w_{n\wedge \tau}) ] \leq &\E [ G(w_0) ]+\left(\frac{3\eta^2\lambda^2}{2}+  \frac{5}{2}\lambda\eta\delta \right)  n^2\eta C_1  \\
	&+\lambda n \eta \left( \|w^*\|^2_{A,\lambda} + \frac52\delta \|w^*\|^2_A\right)  +\frac{3n\eta^2\|A\|}2r^2, 
	\end{align*}
	together with 
	\begin{align*}
	&\left(\frac{3\eta^2\lambda^2}{2}+  \frac{5}{2}\lambda\eta\delta \right)N^2\eta=\tildeO(\epsilon^{\frac{2c+1}{c}} (\log \epsilon)^2), \ \lambda N\eta=\tildeO(|\log \epsilon|),\\
	& \|w^*\|^2_{A,\lambda} + \frac52\delta \|w^*\|^2_A=O(\epsilon),\ N\eta^2 \frac{3\|A\| r^2}{2}=\tildeO(\epsilon |\log \epsilon|),
	\end{align*}
	we have
	\begin{align*}
	\E [ G(w_{N\wedge \tau}) ] \leq &\E [ G(w_0) ]+ \tildeO(\epsilon |\log \epsilon| ). 
	\end{align*}
}	
\blue{	Given $\E [ G(w_{N\wedge \tau}) ] \leq \E [ G(w_0) ]+ \tildeO(\epsilon |\log \epsilon|)$ for both cases, according to Chebyshev's inequality, we have  
	\begin{align*}
	&\Prob(\tau\leq N) = \Prob(\{w_{N\wedge \tau} \notin \calD\}) = \Prob(\{ G(w_{N\wedge \tau}) > (1+a)\E[G(w_0)] \}) \\
	& \leq  \frac{\E [G(w_{N\wedge \tau})] }{ (1+a)\E[G(w_0)]} \leq \frac{ \E [ G(w_0) ]+ \tildeO(\epsilon |\log \epsilon| ) }{ (1+a)\E[G(w_0)]} \leq \frac{1}{1+a} + \tildeO(\epsilon |\log \epsilon| ).
	\end{align*} }
\end{proof}

\section{Proofs of Results for Overparameterization in Statistical Models}
\subsection{Linear regression}
\begin{proof}[Proof of Proposition \ref{prop:linregression}]
	It is straightforward to find the gradient and Hessian of $F$ as:
	\begin{align}
	\nabla F(w)= \Sigma(w-w^*) , \quad \nabla^2 F(w)=\Sigma.
	\end{align}
	This leads to $A=\Sigma, \delta=0, \calD=\reals^p$. 
	
	Next, note that
	\begin{align*}
	\nabla f(w,\zeta) =(x^{T}w - y)x  =  (x^{T} (w-w^*)  -\xi)x.
	\end{align*}
	By Cauchy Schwarz inequality, we have
	\begin{align}
	\notag
	\E \|\nabla f(w,\zeta)-\nabla F(w)\|^2&\leq \mathbb{E}[\|\nabla f(w,\zeta)\|^2] \\
	\notag
	&\leq 2\mathbb{E} [\| xx^{T} (w-w^*) \|^2 ] +2 \mathbb{E}[\|x\xi\|^2] \\
	\label{ineq:ga_n}
	& = 2(w-w^*)^{T}\mathbb{E}[xx^{T}xx^{T}] (w-w^*) + 2\sigma^2\tr(\Sigma).
	\end{align}
	Next we compute $\mathbb{E}[xx^{T}xx^{T}]$. Since $x \sim \mathcal{N}(0, \Sigma)$, it can be decomposed as $x = \Sigma^{1/2}z $ with $z \sim \mathcal{N}(0, I_p)$. Let the  eigen-decomposition of $\Sigma$  be $V^{T}\Lambda V$ and denote
	$\Sigma^{1/2} = V^{T}\Lambda^{1/2} V$. We notice that $z'=Vz \sim \mathcal{N}(0, I_p)$, then the $(i,j)$-th element of $Vzz^{T}V^{T} \Lambda Vzz^{T}V^{T}$ is $\sum_{k=1}^{p} \lambda_k z'_iz'_j(z'_k)^2$, and taking expectation results in 
	\begin{align*}
	\mathbb{E}[ Vzz^{T}V^{T} \Lambda Vzz^{T}V^{T}] = \text{diag}\left[2\lambda_1+\sum_{j=1}^{p}\lambda_j,..., 2\lambda_{j}+\sum_{j=1}^{p}\lambda_j,...,2\lambda_{p}+\sum_{j=1}^{p}\lambda_j\right].
	\end{align*}
	Thus we have
	\begin{align*}
	\mathbb{E}[xx^{T}xx^{T}]& = V^{T}\Lambda^{1/2} \mathbb{E}[ Vzz^{T}V^{T} \Lambda Vzz^{T}V^{T}] \Lambda^{1/2} V \\
	& =  V^{T}\Lambda^{1/2} \text{diag}\Big[2\lambda_1+\sum_{j=1}^{p}\lambda_j,..., 2\lambda_{j}+\sum_{j=1}^{p}\lambda_j,...,2\lambda_{p}+\sum_{j=1}^{p}\lambda_j\Big] \Lambda^{1/2} V \\
	&\preceq 3\tr(\Sigma)\Sigma . 
	\end{align*}
	Plugging  this upper bound in (\ref{ineq:ga_n}) gives us 
	\[
	\E \|\nabla f(w,\zeta)-\nabla F(w)\|^2\leq  6\tr(\Sigma)\|w-w^*\|^2_\Sigma + 2\sigma^2\tr(\Sigma).
	\]
	Finally, we note that $\|w-w^*\|^2_\Sigma=2G(w)$, and by Young's inequality
	\[
	\|w-w^*\|^2_\Sigma \leq 2\|w\|_\Sigma^2+2\|w^*\|^2_\Sigma\leq2\|\Sigma\|\|w\|^2 +2\|w^*\|^2_\Sigma,
	\]
	\[
	\|w-w^*\|^2_\Sigma=(w-w^*)^T\Sigma (w-w^*)=\frac12 (w-w^*)^T\nabla F(w).
	\]
	Therefore, we conclude that 
	\[
	\E \|\nabla f(w,\zeta)-\nabla F(w)\|^2\leq 2\sigma^2\tr(\Sigma)+12\tr(\Sigma)\|w^*\|^2_\Sigma+12\tr(\Sigma)\min\{G(w),\|\Sigma\| \|w\|^2\}. 
	\]
\end{proof}
\begin{rem}
	\label{rem:4thmoment}
	In the proof above, we used the Gaussian distribution assumption only to obtain the first, second and fourth moments of $x$. This proof can be extended to scenarios where  $x$ has a non-Gaussian distribution, as long as an upper bound of $\E[xx^Txx^T]$ is available. Similar extensions can be made for other proofs below as well. 
\end{rem}

\subsection{Logistic regression}

\begin{proof}[Proof for Proposition \ref{prop:logit}]
	By Fubini's theorem, 
	\begin{align*}
	\nabla F(w)& = \mathbb{E}\nabla f(w, \zeta) = \mathbb{E}\dfrac{-yx}{1+\exp(yx^{T}w)},
	\end{align*}
	and 
	\begin{align*}
	\nabla^2 F(w)& = \mathbb{E} \nabla\dfrac{-yx}{1+\exp(yx^{T}w)} = \mathbb{E} \dfrac{y^2\exp(yx^{T}w)xx^{T}}{(1+\exp(yx^{T}w))^2}.
	\end{align*}
	Because $0 < \dfrac{y^2\exp(yx^{T}w)}{(1+\exp(yx^{T}w))^2} < 1$ and $0 \preceq xx^T$, we find  $0 \preceq \nabla^2 F(w) \preceq \Sigma$.
	
	Next, we observe
	\begin{align*}
	\nabla f(w,\zeta) = \dfrac{-yx}{1+\exp(yx^{T}w)}.
	\end{align*}
	Then, because $y=\pm 1$, we obtain
	\begin{align*}
	\mathbb{E} [\|\nabla f(w,\zeta)\|^2] = \mathbb{E} \left[\left(\dfrac{-y}{1+\exp(yx^{T}w)}\right)^2\|x\|^2\right] \leq \mathbb{E}[\|x\|^2]= \tr(\Sigma).
	\end{align*}
\end{proof}

\subsection{$M$-estimator with Tukey's biweight loss}

\begin{proof}[Proof for Proposition \ref{prop:Mest}]
	First of all, let $v=w-w^*, u=x^Tv-\xi$. We find that 
	\[
	\nabla f(w,\zeta)=(1-(u/c)^2)^2ux1_{|u|\leq c}.
	\]
	Then, by Fubini theorem, we have
	\begin{align*}
	\nabla F(w) &=\nabla_v F(w)=\E \nabla [\rho(x^T v - \xi)] =  \E[(1-(u/c)^2)^2u x 1_{\{ |u| \leq c\}}  ],\\
	\nabla^2 F(w)  &= \E [xx^T(1-(u/c)^2)(1-5(u/c)^2)1_{|u|\leq c} ]. 
	\end{align*}
	For the first two claims, note that
	\begin{align*}
	&\E \|\nabla f(w,\zeta)\|^2=\E[(1-(u/c)^2)^4(u/c)^21_{|u|\leq c}\|x\|^2] \leq \E \|x\|^2=\tr(\Sigma),\\
	&\nabla^2 F(w)= \E [xx^T(1-(u/c)^2)(1-5(u/c)^2)1_{|u|\leq c}]\preceq \E xx^T=\Sigma.
	\end{align*}
	At $w=w^*$,
	\begin{align*}
	\nabla^2 F(w^*) &= \E [xx^T(1-(\xi/c)^2)(1-5(\xi/c)^2)1_{|\xi|\leq c}]=c_0\Sigma. 
	\end{align*}
	We consider the directional derivative along the $v$ direction
	\begin{align*}
	\langle  v, \nabla^3 F(w)\rangle &:=\lim_{\epsilon\to 0}\frac{1}{\epsilon}(\nabla^2 F(w+\epsilon v) -\nabla^2 F(w) )\\
	& = \lim_{\epsilon\to 0}\frac{1}{\epsilon}(  \E [xx^T(1-(u/c + \epsilon x^{T}v)^2)(1-5(u/c+\epsilon x^{T}v)^2)1_{|u|\leq c}] \\
	&\quad \quad \quad  \quad -  \E [xx^T(1-(u/c)^2)(1-5(u/c)^2)1_{|u|\leq c}]) \\
	& = \E [4xx^Tx^Tv(3u/c- 5(u/c)^3) 1_{|u|\leq c}].
	\end{align*}
	We find 
	\begin{align*}
	\pm \langle  v, \nabla^3 F(w)\rangle&= \pm \E [4xx^Tx^T v(3u/c- 5(u/c)^3) 1_{|u|\leq c}]\\
	&\preceq \E [4xx^T|x^T v ||5(u/c)^3-3u/c|1_{|u|\leq c}]\preceq \E [8xx^T|x^T v|].
	\end{align*}
	For any test vector $\psi$, 
	\begin{align*}
	|\psi^T\langle  v, \nabla^3 F(w)\rangle\psi|&\leq 8\E [(x^T \psi)^2 |x^T v |] \leq 8\sqrt{\E [(x^T \psi)^4] \E[ |x^T v|^2] }\\
	&= 8\sqrt{3(\psi^T\Sigma \psi)^2 (v^T\Sigma v)}\leq 16\|v\|_\Sigma \psi^T\Sigma \psi.
	\end{align*}
	Therefore,
	\[
	-16\|v\|_\Sigma \Sigma \preceq \langle  v, \nabla^3 F(w)\rangle\preceq 16\|v\|_\Sigma \Sigma. 
	\]
	Furthermore, since $w=w^*+v$, from
	\[
	\nabla^2 F(v+w^*)=\nabla^2 F(w^*)+\int^1_0\langle  v, \nabla^3 F(w^*+sv)\rangle ds,
	\]
	we find 
	\[
	\nabla^2 F(w)\succeq -\delta \Sigma,
	\]
	if $16\|v\|_\Sigma\leq c_0+\delta.$
\end{proof}

\subsection{Two-layer neural network}
First of all, we provide a simple upper bound when computing the fourth order moments of Gaussian random variables. 
\begin{lem}
	\label{lem:Gs4th}
	If $x\in \reals^p$ is Gaussian with mean being zero, for any PSD $A\in \reals^{p\times p}$ and $a>0$
	\[
	\E (x^TAx+a)^2\leq 3(\E (x^TAx+a))^2.
	\]
\end{lem}
\begin{proof}
	Let $\Sigma$ be the covariance matrix of $x$. Since replacing $x$ with $\Sigma^{-1/2}x$, the statement of the Lemma remains the same, therefore we can assume $x\sim \mathcal{N}(0, I_p)$.
	Let $A=V^T \Lambda V$ be the eigenvalue decomposition of $A$, and the eigenvalues of $A$ be $\lambda_1,\ldots,\lambda_p$. Let $z=Vx\sim \mathcal{N}(0,I_p)$. Note that 
	\[
	\E (x^TAx+a)^2=\E (\|z\|_\Lambda^4+2a \|z\|_\Lambda^2+a^2),
	\]
	and further,
	\[
	\E \|z\|_\Lambda^4=\sum_{i,j} \lambda_i\lambda_j\E (z^2_iz_j^2)\leq 3\sum_{i,j} \lambda_i\lambda_j\E z^2_i\E z_j^2=3(\E\|z\|_\Lambda^2)^2. 
	\]
	As a consequence, we obtain
	\[
	\E (x^TAx+a)^2\leq 3(\E \|z\|_\Lambda^2+a)^2=3(\E (x^TAx+a))^2.
	\]
\end{proof}
\begin{lem}
	\label{lem:2NNcomp}
	Assume that $\psi(0)=0$ and $|\dot{\psi}|,|\ddot{\psi}|\leq C$. Denote
	\[
	\Sigma^{\star}=\text{diag}\{I_k,\Sigma,\cdots, \Sigma, I_k\}\in\reals^{(p+2)k\times (p+2)k},
	\]
	and $\Delta w=w-w^*$. Then the followings hold
	\begin{enumerate}[1)]
		\item $\E\|\nabla f(w)\|^2 \leq 8\sqrt{3}(1+\tr(\Sigma))(6C^2\|\Delta w\|^2_{\Sigma^\star}( \|w^*\|_{\Sigma^\star}^2+\|w\|_{\Sigma^\star}^2)+\sigma_0^2)C^2\|w\|^2_{\Sigma^\star}$.
		\item  $\E \nabla g(w,x) \nabla g(w,x)^T \preceq 6C^2\|w\|^2_{\Sigma^\star} \Sigma^\star$.
		\item $-M_w\preceq \E (g(w,x)-g(w^*,x)-\xi)\nabla^2 g\preceq M_w$, where
		\[
		M_w:=6\sqrt{2} C^2(\|c\|_\infty+1) \|\Delta w\|_{\Sigma^\star}(\|w^*\|_{\Sigma^\star}+\|w\|_{\Sigma^\star})\Sigma^\star,
		\]
		with $\|c\|_{\infty} := \max \limits_{i}\{|c_i| \}$. 
		\item $G(w)\leq 6C^2\|\Delta w\|^2_{\Sigma^\star}( \|w^*\|_{\Sigma^\star}^2+\|w\|_{\Sigma^\star}^2).$
	\end{enumerate}

\end{lem}

\begin{proof}
	For simplicity of discussion, we denote $z_i=b_i^Tx+a_i$ and $z=bx +a$.\\
	\textbf{Proof for Claim 1)} We note that $\nabla f(w)=2(g(w,x)-g(w^*,x)-\xi)\nabla g(w,x)$, thus
	\begin{align}
	\notag
	\E\|\nabla f(w)\|^2 &= 4\E[ (g(w,x)-g(w^*,x) )^2 \|\nabla g(w,x)\|^2 ] +  4 \sigma_0^2\E\|\nabla g(w,x)\|^2\\
	\label{eqn:bound_f'}
	&\leq 4\sqrt{\E (g(w,x)-g(w^*,x) )^4}\sqrt{\E \|\nabla g(w,x)\|^4} +  4 \sigma_0^2\sqrt{\E\|\nabla g(w,x)\|^4}. 
	\end{align} 
	Note that
	\[
	\nabla g=
	\begin{bmatrix}
	c\circ \dot{\psi}(z);
	c_1\dot{\psi}(z_1)x;
	\cdots;
	c_k\dot{\psi}(z_k)x;
	\psi(z)
	\end{bmatrix}^T\in \reals^{2k+kp},
	\]
	as a consequence, we have
	\begin{align}
	\notag
	\E\|\nabla g(w,x)\|^4 &= \E\left( \| c\circ \dot{\psi}(z) \|^2 +\sum_{i=1}^{k}  \|c_i\dot{\psi}(z_i)x\|^2 +  \sum_{i=1}^k\|\psi(z_i)\|^2 \right)^2\\
	\notag
	& \leq \E \left(C^2 \|c\|^2 + \sum_{i=1}^{k} (c_i)^2C^2\|x\|^2 + 2C^2\sum_{i=1}^k (b_i^Tx)^2+2C^2\|a\|^2\right)^2\\
	\notag
	&\mbox{Since $x$ is mean zero Gaussian, by Lemma \ref{lem:Gs4th}}\\
	\notag
	& \leq 3\left( C^2 \|c\|^2 + \sum_{i=1}^{k} (c_i)^2C^2\E\|x\|^2 + 2C^2\E\sum_{i=1}^k (b_i^Tx)^2+2C^2\|a\|^2\right)^2\\
	\notag
	&\leq 3C^4\left( \|c\|^2 + \|c\|^2\tr(\Sigma) + 2\sum_{i=1}^k\|b_i\|^2_\Sigma+2\|a\|^2\right)^2\\
	\label{eqn:nabg4}
	&\leq 12C^4(1+\tr(\Sigma))^2 \|w\|^4_{\Sigma^\star}.
	\end{align}
	Next, we let $w^s=sw+(1-s)w^*$ and $C_w^2=\|w\|^2_{\Sigma^\star}+\|w^*\|^2_{\Sigma^\star}$.
	By the convexity of $\|\,\cdot\,\|_{\Sigma^\star}^2$, we get
	\[
	\|w^s\|_{\Sigma^\star}^4\leq \max\{\|w\|_{\Sigma^\star}^4,\|w^*\|_{\Sigma^\star}^4\}\leq \|w\|_{\Sigma^\star}^4+ \|w^*\|_{\Sigma^\star}^4\leq C_w^4.
	\]
	Then, we have
	\begin{align}
	\notag
	&|g(w,x)-g(w^*,x)|^2=\left(\int^1_0 \Delta w^T\nabla g(w^s,x) ds\right)^2\\
	\notag
	&\leq \int^1_0\left(\Delta a^T c^s\circ  \dot{\psi}(z^s)+\sum_{i=1}^k c^s_i\dot{\psi}(z^s_i)\Delta b_i^T x+\Delta c^T\psi(z^s) \right)^2ds\\
	\notag
	&\leq \int^1_0\left(C\|\Delta a\|\|c^s\| +C\sum_{i=1}^k |c^s_i||\Delta b_i^T x|+\|\Delta c\|\|\psi(z^s)\|\right)^2ds\\
	\notag
	&\leq \int^1_0 \left(C_w^2\|\Delta a\|^2 +C_w^2\sum_{i=1}^k |\Delta b_i^T x|^2+\|\Delta c\|^2\|\psi(z^s)\|^2/C^2\right)ds\\
	\notag
	&\quad\quad \cdot\int^1_0  \left(\frac{C^2\|c^s\|^2}{C^2_w} +\frac{C^2}{C_w^2}\sum_{i=1}^k |c^s_i|^2+C^2\right)ds\\
	\notag
	&\leq \int^1_0\left(C_w^2\|\Delta a\|^2 +C_w^2\sum_{i=1}^k |\Delta b_i^T x|^2+\|\Delta c\|^2\|\psi(z^s)\|^2/C^2\right)ds\\
	\notag
	&\quad\quad \cdot \int^1_0\left(2C^2\|c^s\|^2/C^2_w +C^2\right)ds\\
	\label{eqn:g-g}
	&\leq 3C^2 \int^1_0\left(C_w^2\|\Delta a\|^2 +C_w^2\sum_{i=1}^k |\Delta b_i^T x|^2+\|\Delta c\|^2\|\psi(z^s)\|^2/C^2\right) ds.  
	\end{align} 
	By Lemma \ref{lem:Gs4th} and $\E|\Delta b_i^T x|^2=\Delta b_i^T\Sigma \Delta b_i$, we have
	\begin{align*}
	&\E \left(C_w^2\|\Delta a\|^2 +C_w^2\sum_{i=1}^k |\Delta b_i^T x|^2+\|\Delta c\|^2\|\psi(z^s)\|^2/C^2\right)^2\\
	&\leq \E \left(C_w^2\|\Delta a\|^2 +C_w^2\sum_{i=1}^k |\Delta b_i^T x|^2+2\|\Delta c\|^2(\|a^s\|^2+\sum_{i=1}^k|(b_i^s)^Tx|^2)\right)^2\\
	& \text{Note that } \|a^{s}\|^2+\sum_{i=1}^k|(b_i^{s})^Tx|^2  \leq \max\{\|w\|_{\Sigma^\star}^2,\|w^*\|_{\Sigma^\star}^2 \}\leq C_w^2 \\
	&\leq \left(C_w^2\|\Delta a\|^2 +C_w^2\sum_{i=1}^k \|\Delta b_i\|^2_\Sigma+2C_w^2\|\Delta c\|^2\right)^2\\
	&\leq 4(\|w\|_{\Sigma^\star}^2+\|w^*\|_{\Sigma^\star}^2)^2\|\Delta w\|^4_{\Sigma^\star}. 
	\end{align*}
	Replace these bounds into the square of \eqref{eqn:g-g}, we find
	\begin{equation}
	\label{eqn:g-g4}
	\E|g(w,x)-g(w^*,x)|^4\leq 36 C^4\|\Delta w\|^4_{\Sigma^\star}( \|w^*\|_{\Sigma^\star}^2+\|w\|_{\Sigma^\star}^2)^2.
	\end{equation}
	Furthermore, we combine this with \eqref{eqn:nabg4} into \eqref{eqn:bound_f'}, we  find that 
	\[
	\E\|\nabla f(w)\|^2 \leq 8\sqrt{3}(1+\tr(\Sigma))(6C^2\|\Delta w\|^2_{\Sigma^\star}( \|w^*\|_{\Sigma^\star}^2+\|w\|_{\Sigma^\star}^2)+\sigma_0^2)C^2 \|w\|^2_{\Sigma^\star}.
	\]
	
	\noindent\textbf{Proof for Claim 2): } Recall that
	\[
	\nabla g=
	\begin{bmatrix}
	c\circ \dot{\psi}(z);
	c_1\dot{\psi}(z_1)x;
	\cdots;
	c_k\dot{\psi}(z_k)x;
	\psi(z)
	\end{bmatrix}\in \reals^{2k+kp}. 
	\]
	With $u \in \reals^{k}, \ 
	v_1 \in \reals^{p}, ...,  
	v_k \in \reals^{p}, \ 
	w \in  \reals^{k}, $ we define
	\[
	W = 
	\begin{bmatrix}
	u;
	v_1 ;
	\cdots;
	v_k;
	w 
	\end{bmatrix}\in \reals^{2k+kp},
	\]
	and show that $W^T\E \nabla g \nabla g^{T} W\preceq   6C^2\|w\|^2_{\Sigma^\star} W^T \Sigma^\star W$. Note that 
	\begin{align}
	\label{eqn:nabnab}
	&W^T\E \nabla g \nabla g^{T} W\\
	\notag
	&=\E u^T[ c\circ \dot{\psi}(z)(c\circ \dot{\psi}(z))^T]u+v_i^T\E[ c_i^2\dot{\psi}(z_i)\dot{\psi}(z_i)xx^T] v_i+w^T\E[ \psi(z)(\psi(z) )^T]w\\
	\notag
	&\quad+2u^T\E c\circ \dot{\psi}(z) (\psi(z) )^T w +2\sum_{i=1}^ku^T\E c\circ \dot{\psi}(z) c_i\dot{\psi}(z_i)x^T v_i\\
	\notag
	&\quad  +2\sum_{i<j}v_i^T\E c_i\dot{\psi}(z_i)c_j\dot{\psi}(z_j)xx^T v_j+2\sum_{i=1}^k v_i^T\E c_i\dot{\psi}(z_i)x (\psi(z) )^T w. 
	\end{align}
	For the diagonal terms, note that 
	\begin{align*}
	\E [ c\circ \dot{\psi}(z)(c\circ \dot{\psi}(z))^T] & \preceq \E [ \|c\circ \dot{\psi}(z) \|^2 I_{k}] \preceq C^2  \|c\|^2 I_{k}, \\
	\E[c_i^2 \dot{\psi}(z_i)\dot{\psi}(z_i)xx^T] & \preceq C^2c_i^2  \E [xx^T]=C^2 c_i^2 \Sigma ,\\
	\E[ \psi(z)(\psi(z) )^T] & \preceq \E[ \|\psi(z)\|^2 I_{k} ] \preceq 2C^2\Big(\|a\|^2+\sum_{i=1}^k \|b_i\|^2_\Sigma\Big) I_{k}.
	\end{align*}
	For the cross terms, note that by Cauchy Schwarz inequality
	\begin{align*}
	u^Tc\circ \dot{\psi}(z) c_i\dot{\psi}(z_i)x^T v_i 
	& \leq |u^Tc\circ \dot{\psi}(z)| |c_i\dot{\psi}(z_i)x^T v_i| \\
	& = (u^Tc\circ \dot{\psi}(z) (c\circ \dot{\psi}(z))^Tu )^{1/2} (v_i^T(c_i\dot{\psi}(z_i))^2xx^T v_i)^{1/2} \\
	& \leq \frac{c_i^2}{2\|c\|^2}(u^Tc\circ \dot{\psi}(z) (c\circ \dot{\psi}(z))^Tu) + \frac{\|c\|^2}{2}( v_i^T(\dot{\psi}(z_i))^2xx^T v_i),
	\end{align*}
	and similarily,
	\begin{align*}
	u^Tc\circ \dot{\psi}(z) (\psi(z) )^T w 
	& \leq \frac{1}{2} u^Tc\circ \dot{\psi}(z)(c\circ \dot{\psi}(z))^Tu + \frac{1}{2} w^T \psi(z) (\psi(z) )^T w,\\
	v_i^Tc_i\dot{\psi}(z_i)c_j\dot{\psi}(z_j)xx^T v_j  & \leq \frac{c_j^2}{2} v_i^T(\dot{\psi}(z_i) )^2xx^T v_i + \frac{c_i^2}{2} v_j^T(\dot{\psi}(z_j))^2xx^T v_j,\\
	v_i^Tc_i\dot{\psi}(z_i)x (\psi(z) )^T w & \leq  \frac{\|c\|^2}{2} v_i^T(\dot{\psi}(z_i) )^2 xx^Tv_i
	+\frac{c_i^2}{2\|c\|^2} w^T(\psi(z) ) (\psi(z) )^T w.
	\end{align*}
	Plugging the  results above into \eqref{eqn:nabnab} gives us 
	\begin{align*}
	&\E \nabla g(w,x) \nabla g(w,x)^T\\
	&\preceq 
	C^2\begin{bmatrix}
	2 \|c\|^2 I_k & \mathbf{0}_{k \times p}&  \mathbf{0}_{k \times p} &  \mathbf{0}_{k \times p} & \mathbf{0}_{k \times k}\\
	\mathbf{0}_{p \times k} & 3\|c\|^2 \Sigma & \mathbf{0}_{p \times p} & \mathbf{0}_{p \times p}  & \mathbf{0}_{p \times k}\\
	\mathbf{0}_{p \times k} & \mathbf{0}_{p \times p} & \ddots &  \mathbf{0}_{p \times p}  & \mathbf{0}_{p \times k} \\
	\mathbf{0}_{p \times k} & \mathbf{0}_{p \times p} & \mathbf{0}_{p \times p}  &  3\|c\|^2 \Sigma & \mathbf{0}_{p \times k}\\
	\mathbf{0}_{k \times k}&  \mathbf{0}_{k \times p} &  \mathbf{0}_{k \times p} &  \mathbf{0}_{k \times p} & 6\left(\|a\|^2+\sum_{i=1}^k \|b_i\|^2_\Sigma\right)I_k\\
	\end{bmatrix}\\
	&\preceq 6C^2\|w\|^2_{\Sigma^\star} \Sigma^{\star}.
	\end{align*}

	\noindent\textbf{Proof for Claim 3): } First of all, we find that
	\begin{align*}
	\nabla^2 g =
	\begin{bmatrix}
	D_{c\circ \ddot{\psi}(z)} &  c_1 \ddot{\psi}(z_1)e_1x^T  & c_2 \ddot{\psi}(z_2)e_2x^T&\cdots &c_k \dot{\psi}(z_k)xe_k^T &D_{\dot{\psi}(z)}\\
	c_1 \ddot{\psi}(z_1)xe_1^T & c_1 \ddot{\psi}(z_1)xx^T &\textbf{0}_{p\times p} &\cdots &\textbf{0}_{p\times p}  &  \dot{\psi}(z_1)xe_1^T \\
	c_2 \ddot{\psi}(z_2)xe_2^T & \textbf{0}_{p\times p}  &c_2 \ddot{\psi}(z_2)xx^T &\cdots &\textbf{0}_{p\times p}  &  \dot{\psi}(z_2)xe_2^T\\
	\vdots & \vdots & & &\vdots & \vdots\\
	c_k \ddot{\psi}(z_k)xe_k^T & \textbf{0}_{p\times p}  &\cdots &\textbf{0}_{p\times p}  &  c_k \ddot{\psi}(z_k)xx^T  &  \dot{\psi}(z_k)xe_k^T \\
	D_{\dot{\psi}(z)} & \dot{\psi}(z_1)e_1x^{T} &\dot{\psi}(z_2)e_2x^T &\cdots & \dot{\psi}(z_k)e_k x^T&\textbf{0}_{k\times k}  
	\end{bmatrix}.
	\end{align*}
	In above, we use $D_{v}$ to denote the diagonal matrix with diagonal entries being components of $v$. 
	We will first show that $\nabla^2 g\preceq Q_x\preceq (2\|c\|_\infty+2)C\Sigma_x^\star$, where
	\begin{align*}
	Q_x &:=
	C\text{diag}\{(2\|c\|_\infty+1) I_k, (2\|c\|_\infty+1) xx^T,\ldots,(2\|c\|_\infty+1) xx^T, 2 I_k\},\\
	\Sigma^\star_x & :=\text{diag}\{ I_k,  xx^T,\ldots, xx^T, I_k\}.
	\end{align*}
	Recall $W = [ u; v_1 ; \cdots; v_k; w ] \in \reals^{2k+kp}$.  Note that 
	\begin{align}
	\notag
	W^T\nabla^2 g W=& u^T D_{c\circ \ddot{\psi}(z)}u+\sum_{i=1}^k c_i \ddot{\psi}(z_i)(v_i^Tx)^2+2w^TD_{\dot{\psi}(z)}u\\
	\label{eqn:WnabW}
	&+2\sum_{i=1}^k c_i \ddot{\psi}(z_i)(v^T_ix) (u^Te_i)+2\sum^k_{i=1} \dot{\psi}(z_i)(v^T_ix) (w^Te_i).
	\end{align}
	And further
	\begin{align*}
	u^{T}D_{c\circ \ddot{\psi}(z)}u & \leq \|D_{c\circ \ddot{\psi}(z)}\|\|u\|^2 \leq C\|c\|_\infty\|u\|^2,\\
	c_i \ddot{\psi}(z_i)(v_i^Tx)^2& \leq C\|c\|_\infty  (v_i^Tx)^2,\\
	2w^TD_{\dot{\psi}(z)}u& \leq C\|w\|^2+C\|u\|^2,\\
	2c_i \ddot{\psi}(z_i)(v^T_ix) (u^Te_i) & \leq \|c\|_\infty C((v^T_ix)^2+(u^Te_i)^2),\\
	2\dot{\psi}(z_i)(v^T_ix) (w^Te_i) &\leq  C((v^T_ix)^2+(w^Te_i)^2). 
	\end{align*}
	Replace these upper bounds to terms in \eqref{eqn:WnabW}, we find 
	\[
	W^T\nabla^2 g W\leq W^TQ_xW,
	\]
	because $\sum_i (u^Te_i)^2=\|u\|^2$. Since this holds for all $W$, we have $\nabla^2 g\preceq Q_x$. Finally, we note that 
	\begin{align*}
	|W^T\E (g(w,x)-g(w^*,x)-\xi)\nabla^2 g W|&=|W^T\E [(g(w,x)-g(w^*,x))\nabla^2 g ]W|\\
	&\leq \sqrt{\E (g(w,x)-g(w^*,x))^2}\sqrt{\E (W^T\nabla^2 g W)^2}. 
	\end{align*}
	Recall \eqref{eqn:g-g4}, we have 
	\[
	\sqrt{\E (g(w,x)-g(w^*,x))^2}\leq (\E (g(w,x)-g(w^*,x))^4)^{1/4}\leq \sqrt{6}C\|\Delta w\|_{\Sigma^\star}( \|w^*\|_{\Sigma^\star}+\|w\|_{\Sigma^\star}).
	\]
	Then, by Lemma \ref{lem:Gs4th}, we have
	\[
	\E (W^T\nabla^2 g W)^2\leq 4(\|c\|_\infty+1)^2C^2\E (W^T\Sigma^\star_xW)^2\leq 12 C^2(\|c\|_\infty+1)^2(W^T\Sigma^\star W)^2.
	\]
	In combination, we find 
	\[
	|W^T\E (g(w,x)-g(w^*,x)-\xi)\nabla^2 g W|\leq 6\sqrt{2}C^2(\|c\|_\infty+1) \|\Delta w\|_{\Sigma^\star}(\|w^*\|_{\Sigma^\star}+\|w\|_{\Sigma^\star})W^T\Sigma^\star W.
	\]
	This verifies our claim 3).

	\noindent\textbf{Proof for Claim 4): } Simply note that by \eqref{eqn:g-g4}, we get
	\[
	G(w)=\E |g(w,x)-g(w^*,x)|^2\leq 6C^2\|\Delta w\|^2_{\Sigma^\star}( \|w^*\|_{\Sigma^\star}^2+\|w\|_{\Sigma^\star}^2).
	\]
	
\end{proof}

\begin{proof}[Proof for Proposition \ref{prop:2LNN}.]
	First, we find that, when $w\in \calD$,
	\begin{align}\label{c_inf}
	\|c\|^2_\infty\leq\|w\|^2_{\Sigma^\star}\leq (1+\tfrac14 )^2\|w^*\|^2_{\Sigma^\star}\leq 2\|w^*\|^2_{\Sigma^\star}. 
	\end{align}
	Note that  
	\begin{align*}
	\nabla^2 F&=\E \nabla g(w,x) \nabla g(w,x)^T+ \E (g(w,x)-g(w^*,x)-\xi)\nabla^2 g(w,x)\\
	&=\E \nabla g(w,x) \nabla g(w,x)^T+ \E (g(w,x)-g(w^*,x))\nabla^2 g(w,x).
	\end{align*}
	By Lemma \ref{lem:2NNcomp} claim 2) and claim 3) and (\ref{c_inf}), we have
	\begin{align*}
	\E \nabla g(w,x) \nabla g(w,x)^T & \preceq 6C^2\|w^*\|^2_{\Sigma^\star}  \Sigma^{\star} , \\
	\E (g(w,x)-g(w^*,x))\nabla^2 g(w,x)&\preceq  6\sqrt{2}C^2(\|c\|_\infty+1) \delta C_1(w^*)\|w^*\|_{\Sigma^\star}( \|w^*\|_{\Sigma^\star}+\|w\|_{\Sigma^\star})\Sigma^\star   \\
	& \preceq 18\sqrt{2}C^2(2\|w^*\|_{\Sigma^\star}+1) \delta C_1(w^*)\|w^*\|_{\Sigma^\star}^2\Sigma^\star\\
	&\preceq 4\delta C^2\|w^*\|^2_{\Sigma^\star} \Sigma^{\star}. 
	\end{align*}
	So $\nabla^2 F\preceq C_0(w^*)\Sigma^{\star}$. Also note that $\E \nabla g(w,x) \nabla g(w,x)^T\succeq 0$, we have
	\[
	\nabla^2F \succeq \E (g(w,x)-g(w^*,x))\nabla^2 g(w,x)\succeq  -4\delta C^2\|w^*\|^2_{\Sigma^\star}  \Sigma^{\star}\succeq -\delta A. 
	\]
	Then, by Lemma \ref{lem:2NNcomp} claim 1) and (\ref{c_inf}), we find that 
	\begin{align*}
	\E \|\nabla f(w)-&\nabla F(w)\|^2\leq \E \|\nabla f(w)\|^2\\
	&\leq  8\sqrt{3}(1+\tr(\Sigma))(6 C^2\|\Delta w\|^2_{\Sigma^\star}( \|w^*\|_{\Sigma^\star}^2+\|w\|_{\Sigma^\star}^2)+\sigma_0^2)C^2 \|w\|^2_{\Sigma^\star}\\
	&\leq 8\sqrt{3}(1+\tr(\Sigma))(18 C^2\|\Delta w\|^2_{\Sigma^\star} \|w^*\|_{\Sigma^\star}^2+\sigma_0^2)C^2 \|w^*\|^2_{\Sigma^\star}\\
	&\leq 8\sqrt{3}(1+\tr(\Sigma))(18 \delta^2 (C_1(w^*))^2 C^2 \|w^*\|_{\Sigma^\star}^2\|w^*\|_{\Sigma^\star}^2+\sigma_0^2)C^2 \|w^*\|^2_{\Sigma^\star}\\
	&\leq 8\sqrt{3}(1+\tr(\Sigma))(C^2 \|w^*\|_{\Sigma^\star}^4+\sigma_0^2)C^2 \|w^*\|^2_{\Sigma^\star}.
	\end{align*}
	Finally, by claim 4) of Lemma \ref{lem:2NNcomp}, when $w_0\in \calD$, we have
	\[
	G(w_0)\leq 18C^2(C_1(w^*))^2 \delta^2 \|w^*\|_{\Sigma^\star}^4\leq C^2 \|w^*\|_{\Sigma^\star}^4. 
	\]
\end{proof}

\vskip 0.2in
\bibliography{spec}

\end{document}